\documentclass[final, 12pt, a4paper]{article}

\usepackage{authblk}

\usepackage{packages}
\usepackage{custom_commands}
\usepackage{colors}

\title{The Berkelmans-Pries Feature Importance Method: A Generic Measure of Informativeness of Features}

\author[1,*]{Joris Pries}
\author[1]{Guus Berkelmans}
\author[2]{Sandjai Bhulai}
\author[1]{Rob van der Mei}
\affil[1]{Centrum Wiskunde \& Informatica, Department of Stochastics, Science Park 123, Amsterdam 1098 XG, Netherlands} 
\affil[2]{Vrije Universiteit, Department of Mathematics, De Boelelaan 1111, Amsterdam 1081 HV, Netherlands}
\affil[*]{Corresponding author: Joris Pries, joris.pries@cwi.nl}

\begin{document}
\setlength{\emergencystretch}{0.67em} 

\maketitle

\DTLloaddb[noheader, keys={thekey,thevalue}]{mydata}{"mydata".dat}

\begin{abstract} \label{feature importance: sec: abstract}


    Over the past few years, the use of machine learning models has emerged as a generic and powerful means for prediction purposes. At the same time, there is a growing demand for \textit{interpretability} of prediction models. To determine which features of a dataset are important to predict a target variable $Y$, a \emph{Feature Importance} (FI) method can be used. By quantifying how important each feature is for predicting $Y$, irrelevant features can be identified and removed, which could increase the speed and accuracy of a model, and moreover, important features can be discovered, which could lead to valuable insights. A major problem with evaluating FI methods, is that the ground truth FI is often unknown. As a consequence, existing FI methods do not give the exact correct FI values. This is one of the many reasons why it can be hard to properly interpret the results of an FI method. Motivated by this, we introduce a new global approach named the \emph{Berkelmans-Pries} FI method, which is based on a combination of \emph{Shapley values} and the \emph{Berkelmans-Pries} dependency function. We prove that our method has many useful properties, and accurately predicts the correct FI values for several cases where the ground truth FI can be derived in an exact manner. We experimentally show for a large collection of FI methods (468) that existing methods do not have the same useful properties. This shows that the Berkelmans-Pries FI method is a highly valuable tool for analyzing datasets with complex interdependencies.

\end{abstract}



\section{Introduction} \label{feature importance: sec: introduction}


How important are you? This is a question that researchers (especially data scientists) have wondered for many years. Researchers need to understand how important a random variable (RV) $X$ is for determining $Y$. Which features are important for predicting the weather? Can indicators be found as symptoms for a specific disease? Can redundant variables be discarded to increase performance? These kinds of questions are relevant in almost any research area. Especially nowadays, as the rise of machine learning models generates the need to demystify prediction models. Altmann et al.~\cite{Altmann2010} state that \enquote{In life sciences, interpretability of machine learning models is as important as their prediction accuracy.} Although this might not hold for all research areas, interpretability is very useful. Knowing how predictions are made and why, is crucial for adapting these methods in everyday life.

Determining \emph{Feature Importance} (FI) is the art of discovering the importance of each feature $X_i$ when predicting $Y$. The following two cases are particularly useful. \begin{enumerate*}[label =(\Roman*)]
    \item Finding variables that are not important: \emph{redundant} variables can be discovered using FI methods. Irrelevant features could degrade the performance of a prediction model due to high dimensionality and irrelevant information \cite{Kira1992}. Eliminating redundant features could therefore increase both the speed and the accuracy of a prediction model.
    \item Finding variables that are important: \emph{important} features could reveal underlying structures that give valuable insights. Observing that variable $X$ is important for predicting $Y$ could steer research efforts into the right direction. Although it is critical to keep in mind that high FI does not mean causation. However, FI values do, for example, \enquote{enable an anaesthesiologist to better formulate a diagnosis by knowing which attributes of the patient and procedure contributed to the current risk predicted} \cite{Lundberg2018surgery}. In this way, an FI method can have really meaningful impact.
\end{enumerate*}

Over the years, many FI methods have been suggested, which results in a wide range of FI values for the same dataset. For example, stochastic methods do not even repeatedly predict the same FI values. This makes interpretation difficult. Examine e.g., a result of Fryer et al.~\cite{Fryer2021}, where one measure assigns an FI of $3.19$ to a variable, whereas another method gives the same variable an FI value of $0.265$. This raises a lot of questions: `Which FI method is correct?', 'Is this variable deemed important?', and more generally `What information does this give us?'. To assess the performance of an FI method, the ground truth should be known, which is often not the case \cite{Aas2019,Hooker2018,Tonekaboni2020,Zhou2020}. Therefore, when FI methods were developed, the focus has not yet lied on predicting the \emph{exact} correct FI values. Additionally, many FI methods do not have desirable properties. For example, two features that contain the same amount of information should get the same FI. We later show that this is often not the case.

To improve interpretability, we introduce a new FI method called \emph{Berkelmans-Pries} FI method, which is based on \emph{Shapley values} \cite{Roth} and the \emph{Berkelmans-Pries} dependency function \cite{BP}. Multiple existing methods already use Shapley values, which has been shown to give many nice properties. However, by \emph{additionally} using the \emph{Berkelmans-Pries} dependency function, even more useful properties are obtained. Notably, we prove that this approach accurately predicts the FI in some cases where the ground truth FI can be derived in an exact manner. By combining \emph{Shapley values} and the \emph{Berkelmans-Pries} dependency function a powerful FI method is created. This research is an important step forward for the field of FI, because of the following reasons:

\begin{itemize}
    \item We introduce a new FI method;
    \item We prove multiple useful properties of this method;
    \item We provide some cases where the ground truth FI can be derived in an exact manner;
    \item We prove for these cases that our FI method accurately predicts the correct FI;
    \item We obtain the largest collection of existing FI methods;
    \item We test if these methods adhere to the same properties, which shows that no method comes close to fulfilling all the useful properties;
    \item We provide Python code to determine the FI values \cite{githubfeatureimportance}.
\end{itemize}

\section{Berkelmans-Pries FI} \label{feature importance: sec: Berkelmans-Pries FI}

Kruskal~\cite{Kruskal} stated that
\enquote{There are infinitely many possible measures of association, and it sometimes seems that almost as many have been proposed at one time or another.} Although this quote was about dependency functions, it could just as well have been about FI methods. Over the years, many FI methods have been suggested, but it remains unclear which method should be used when and why \cite{Hooker2018}. In this section, we propose yet another new FI method named the \emph{Berkelmans-Pries FI method} (\bpfi{}). Although it is certainly subjective what it is that someone wants from an FI method, we show in \Cref{feature importance: sec: Properties of BP-FI} that \bpfi{} has many useful and intuitive properties. The \bpfi{} method is based on two key elements: \begin{enumerate*}[label =(\arabic*)]
    \item \emph{Shapley values} and
    \item the \emph{Berkelmans-Pries dependency function}.
\end{enumerate*}
We will discuss these components first to clarify how the \bpfi{} method works.

\subsection{Shapley value approach} \label{feature importance: subsec: Shapley value approach}


The \emph{Shapley value} is a unique game-theoretical way to assign \emph{value} to each \emph{player} participating in a multiplayer \emph{game} based on four axioms \cite{Roth}. This concept is widely used in FI methods, as it can be naturally adapted to determine how \emph{important} (value) each \emph{feature} (player) is for \emph{predicting a target variable} (game). Let $\nvariables$ be the number of features, then the Shapley value of feature $i$ is defined by \begin{align}
    \phi_i(v) & = \sum_{S\subseteq \{1,\dots, \nvariables\} \setminus \{i\}} \frac{|S|! \cdot (\nvariables - |S| - 1)!}{\nvariables!} \cdot \left ( v(S\cup \{i\}) - v(S) \right), \label{feature importance: eq: shapley value function}
\end{align}
where $v(S)$ can be interpreted as the `worth' of the coalition $S$ \cite{Roth}. The principle behind this formulation can also be explained in words: For every possible sequence of features up to feature $i$, the added value of feature $i$ is the difference between the worth before it was included (i.e., $v(S)$) and after (i.e., $v(S\cup\{i\})$). Averaging these added values over all possible sequences of features gives the final Shapley value for feature $i$.

\paragraph*{SHAP}
There are multiple existing FI methods that use Shapley values \cite{Shap,Fryer2021,Datta2016}, which immediately ensures some useful properties. The most famous of these methods is SHAP \cite{Shap}. This method is widely used for \emph{local} explanations (see \Cref{feature importance: subsec: Alternative FI methods}). To measure the local FI for a specific sample $x$ and a prediction model $f$, the \emph{conditional expectation} is used as characteristic function (i.e., $v$ in \Cref{feature importance: eq: shapley value function}). Let $x = (x_1,x_2, \dots, x_{\nvariables})$, where $x_i$ is the feature value of feature $i$, then SHAP FI values can be determined using:
\begin{align}
    v_x(S) := \EE_{z} \left [f(z) \vert z_i = x_i \text{ for all } i \in S, \text{ where } z = (z_1, \dots, z_{\nvariables})
        \right ]. \label{feature importance: eq: SHAP method}
\end{align}
\aanval{waar wordt de verwachting over genomen?}
\verdediging{over alle mogelijke samples(?)}
Observe that the characteristic function $v_x$ is defined locally for each $x$.
To get \emph{global} FI values, an average can be taken over all local FI values. Our novel FI method uses a different characteristic function, namely the \emph{Berkelmans-Pries} dependency function. This leads to many additional useful properties. Furthermore, the focus of this research is not on \emph{local} explanations, but \emph{global} FI values.

\subsection{Berkelmans-Pries dependency function} \label{feature importance: subsec: Berkelmans-Pries dependency function}


A new dependency measure, called the \emph{Berkelmans-Pries} (BP) dependency function, was introduced in \cite{BP}, which is used in the formulation of the \bpfi{} method. It is shown that the \bpdep{} satisfies a list of desirable properties, whereas existing dependency measures did not. It has a measure-theoretical formulation, but this reduces to a simpler and more intuitive version when all variables are discrete \cite{BP}. We want to highlight this formulation to give some intuition behind the \bpdep. It is given by
\begin{align}
    \def\arraystretch{1.5}
    \Dep{X}{Y}:= \left \{ \begin{array}{cl}
                              \frac{\UD{X}{Y}}{\UD{Y}{Y}} & \text{if $Y$ is not a.s. constant,} \\[1ex]
                              \text{undefined}            & \text{if $Y$ is a.s. constant,}
                          \end{array}\right .
    \label{feature importance: eq: definition bp dependency}
\end{align}
where (in the discrete case) it holds that
\begin{align}
    \UD{X}{Y}:=\sum_{x} p_X(x) \cdot \sum_{y} \left \vert p_{Y\vert X=x}(y)   - p_Y(y) \right\vert . \label{feature importance: eq: definition discrete ud}
\end{align}
The \bpd{} measure can be interpreted in the following manner. The numerator is the expected absolute difference between the distribution of $Y$ and the distribution of $Y$ given $X$. If $Y$ is highly dependent on $X$, the distribution changes as knowing $X$ gives information about $Y$, whereas if $Y$ is independent of $X$, there is no difference between these two distributions. The denominator is the maximal possible change in distribution of $Y$ for any variable, which is used to standardize the dependency function. Note that the \bpdep{} is \emph{asymmetric}: $\Dep{X}{Y}$ is the dependency of $Y$ on $X$, not vice versa. Due to the many desirable properties, the \bpdep{} is used for the \bpfi{}.

\subsection{Berkelmans-Pries FI method} \label{feature importance: subsec: Berkelmans-Pries FI method}


One crucial component of translating the game-theoretical approach of Shapley values to the domain of FI is choosing the function $v$ in \Cref{feature importance: eq: shapley value function}. This function assigns for each set of features $S$ a value $v(S)$ that characterizes the `worth' of the set $S$. How this function is defined, has a critical impact on the resulting FI. We choose to define the `worth' of a set $S$ to be the \bpd{} of $Y$ on the set $S$, which is denoted by $\Dep{S}{Y}$ \cite{BP}. Here, $\Dep{S}{Y}=\Dep{Z_S(\DD)}{Y}$ where $\DD$ denotes the entire dataset with all features and $Z_S(\DD)$ is the reduction of the dataset to include only the subset of features $S$. Let $\FISET$ be the set of all feature variables. Now, for every $S\subseteq \FISET$, we define:
\begin{align}
    v(S):= \Dep{S}{Y}. \label{feature importance: eq: value function bp-fi}
\end{align}
In other words, the value of set $S$ is exactly how \emph{dependent} the target variable $Y$ is on the features in $S$. The difference $v(S\cup\{i\}) -v(S)$ in \Cref{feature importance: eq: shapley value function} can now be viewed as the increase in dependency of $Y$ on the set of features, when feature $i$ is also known. The resulting Shapley values using the \bpdep{} as characteristic function are defined to be the \bpfi{} outcome. For each feature $i$, we get:
\begin{align}
    \FI{i} & := \sum_{S\subseteq \FISET \setminus \{i\}} \frac{|S|! \cdot (\nvariables - |S| - 1)!}{\nvariables!} \cdot \left ( v(S\cup \{i\}) - v(S) \right) \nonumber                                                                       \\[2ex]
           & \hphantom{:}= \sum_{S\subseteq \FISET \setminus \{i\}} \frac{|S|! \cdot (\nvariables - |S| - 1)!}{\nvariables!} \cdot \left ( \Dep{S\cup \{i\}}{Y} - \Dep{S}{Y} \right). \label{feature importance: eq: formal definition bp-fi}
\end{align}

Abbreviated notation improves readability of upcoming derivations, which is why we define
\begin{align*}
    w(S, \nvariables) & :=\frac{|S|! \cdot (\nvariables - |S| - 1)!}{\nvariables!}, \label[notations]{feature importance: eq: definition w} \tag{N1} \\[2ex]
    D(X,Y,S)          & := \Dep{S\cup \{X\}}{Y} - \Dep{S}{Y}. \label[notations]{feature importance: eq: definition D} \tag{N2}
\end{align*}

Note that when $Y$ is \emph{almost surely constant} (i.e., $\PP(Y=y)=1$), $\Dep{S}{Y}$ is undefined for any feature set $S$ (see \Cref{feature importance: eq: definition bp dependency}). We argue that it is natural to assume that $\FI{i}$ is also undefined, as every feature attributes everything and nothing at the same time. In the remainder of this paper, we assume that $Y$ is not a.s. constant.

\section{Properties of \bpfi{}} \label{feature importance: sec: Properties of BP-FI}


Recall that it is hard to evaluate FI methods, as the ground truth FI is often unknown \cite{Aas2019,Hooker2018,Tonekaboni2020,Zhou2020}. With this in mind, we want to show that the \bpfi{} method has many desirable properties. We also give some synthetic cases where the \bpfi{} method gives a natural expected outcome. The \bpfi{} method is stooled on \emph{Shapley values}, which are a unique solution based on four axioms \cite{Winter}. These axioms already give many characteristics that are preferable for an FI method. Additionally, using the \bpdep{} ensures that it has extra desirable properties. In this section, we prove properties of the \bpfi{} method and discuss why these are relevant and useful.

\begin{property}[Efficiency] \label{feature importance: prop: efficiency}
    The sum of all FI scores is equal to the total dependency of $Y$ on all features:
    \begin{align*}
        \sum_{i \in \FISET} \FI{i} = \Dep{\FISET}{Y}.
    \end{align*}
\end{property}

\begin{proof}
    Shapley values are \emph{efficient}, meaning that all the value is distributed among the players. Thus, \begin{align*}
        \sum_{i \in \FISET} \FI{i} & = v(\FISET) = \Dep{\FISET}{Y}. \qedhere
    \end{align*}
\end{proof}

\begin{relevance*}
    With our approach, we try to answer the question `How much did each feature contribute to the total dependency?'. The total `payoff' is in our case the total dependency. It is therefore natural to divide the entire payoff (but not more than that) amongst all features.
\end{relevance*}

\begin{corollary} \label{feature importance: corollary: sum stable}
    If adding a RV $X$ to the dataset does not give any additional information (i.e., $ \Dep{\FISET \cup X}{Y} =  \Dep{\FISET}{Y}$), then the sum of all FI remains the same.
\end{corollary}

\begin{proof}
    This directly follows from \Cref{feature importance: prop: efficiency}.
\end{proof}

\begin{relevance*}
    If the collective knowledge remains the same, the same amount of credit is available to be divided amongst the features. Only when new information is added, an increase in combined credit is warranted. A direct result of this corollary is that adding a \emph{clone} (i.e., $X^{\text{clone}}:= X$) of a variable $X$ to the dataset will never increase the total sum of FI.
\end{relevance*}

\begin{property}[Symmetry] \label{feature importance: prop: symmetry}
    If for every $S\subseteq \FISET \setminus \{i, j\}$ it holds that $\Dep{S\cup \{i\}}{Y} = \Dep{S\cup \{j\}}{Y}$, then $\FI{i} = \FI{j}.$
\end{property}

\begin{proof}
    Shapley values are \emph{symmetric}, meaning that if ${v(S\cup \{i\}) = v(S\cup \{j\})}$ for every $S\subseteq \FISET \setminus \{i,j\}$, it follows that $\FI{i} = \FI{j}$. Thus, it automatically follows that \bpfi{} is also symmetric.
\end{proof}

\begin{relevance*}
    If two variables are interchangeable, meaning that they \emph{always} contribute equally to the dependency, it is only sensible that they obtain the same FI. This is a desirable property for an FI method, as two features that contribute equally should obtain the same FI.
\end{relevance*}

\begin{property}[Range]\label{feature importance: prop: range}
    For any RV $X$, it holds that $\FI{X} \in [0,1]$.
\end{property}

\begin{proof}
    The \bpdep{} is \emph{non-increasing} under functions of $X$ \cite{BP}, which means that for any measurable function $f$ it holds that
    \begin{align*}
        \Dep{f(X)}{Y} \leq \Dep{X}{Y}.
    \end{align*}
    Take $f:= Z_S$, which is the function that reduces $\DD$ to the subset of features in $S$. Using the non-increasing property of \bpdep, it follows that:
    \begin{equation}\label{feature importance: eq: dependency increasing}
        \begin{split}
            \Dep{S}{Y} & =\Dep{Z_S(\DD)}{Y}=\Dep{Z_S(Z_{S\cup\{i\}}(\DD))}{Y}  \\[2ex]          &\leq \Dep{Z_{S\cup\{i\}}(\DD)}{Y}=\Dep{S\cup\{i\}}{Y}.
        \end{split}
    \end{equation}
    Examining \Cref{feature importance: eq: formal definition bp-fi}, we observe that every FI value must be greater or equal to zero, as $\Dep{S \cup \{i\}}{Y} -  \Dep{S}{Y} \geq 0.$

    One of the properties of the \bpdep{} is that for any $X,Y$ it holds that $\Dep{X}{Y} \in [0,1]$ \cite{BP}. Using \Cref{feature importance: prop: efficiency}, the sum of all FI values must therefore be in $[0,1]$, as $\sum_{i \in \FISET} \FI{i} = \Dep{\FISET}{Y} \in [0,1] $. This gives an upper bound for the FI values, which is why we can now conclude that $\FI{X} \in [0,1]$ for any RV $X$.
\end{proof}

\begin{relevance*}
    It is essential for interpretability that an FI method is bounded by known bounds. For example, an FI score of 4.2 cannot be interpreted properly, when the upper or lower bound is unknown.
\end{relevance*}

\begin{property}[Bounds] \label{feature importance: prop: bounds}
    Every $\FI{X}$ with $X\in \FISET$ is bounded by \begin{align*}
        \frac{\Dep{X}{Y}}{\nvariables} \leq \FI{X} \leq \Dep{\FISET}{Y}.
    \end{align*}
\end{property}

\begin{proof}
    The upper bound follows from \Cref{feature importance: prop: efficiency,feature importance: prop: range}, as \begin{align*}
        \Dep{\FISET}{Y} = \sum_{i \in \FISET} \FI{i} \geq \FI{X},
    \end{align*}
    where the last inequality follows since $\FI{i} \in [0,1]$ for all $i\in \FISET$.

    The lower bound can be established using the inequality from \Cref{feature importance: eq: dependency increasing} within \Cref{feature importance: eq: formal definition bp-fi}. This gives (using \Cref{feature importance: eq: definition w})
    \begin{align*}
        \FI{X} & = \sum_{S\subseteq \FISET \setminus \{X\}}  w(S, \nvariables) \cdot \Bigl ( \Dep{S\cup \{X\}}{Y}  - \Dep{S}{Y}\Bigr ) \\[2ex]
               & \geq  w(0, \nvariables)  \cdot \left ( \Dep{\emptyset \cup \{X\}}{Y} - \Dep{\emptyset}{Y} \right)                     \\[2ex]
               & = \frac{0! \cdot (\nvariables - 0 - 1)!}{\nvariables!} \cdot \Dep{X}{Y}                                               \\[2ex]
               & = \frac{\Dep{X}{Y}}{\nvariables}. \qedhere
    \end{align*}
\end{proof}

\begin{relevance*}
    These bounds are useful for upcoming proofs.
\end{relevance*}

\begin{property}[Zero FI] \label{feature importance: prop: zero FI}
    For any RV $X$, it holds that \begin{align*}
        \FI{X} = 0 \Leftrightarrow \Dep{S\cup \{X\}}{Y} = \Dep{S}{Y} \text{ for all } S \in \FISET\setminus \{X\}.
    \end{align*}
\end{property}

\begin{proof}
    $\Leftarrow$: When $\Dep{S\cup \{X\}}{Y} = \Dep{S}{Y}$ for all $S \in \FISET \setminus \{X\}$, it immediately follows from \Cref{feature importance: eq: formal definition bp-fi} (with \Cref{feature importance: eq: definition w}) that
    \begin{align*}
        \FI{X} & =
        \sum_{S\subseteq \FISET \setminus \{X\}}  w(S, \nvariables)  \cdot \bigl ( \Dep{S\cup \{X\}}{Y} - \Dep{S}{Y} \bigr )
        \\[2ex]
               & = \sum_{S\subseteq \FISET \setminus \{X\}} \frac{|S|! \cdot (\nvariables - |S| - 1)!}{\nvariables!} \cdot 0 \\[2ex]
               & = 0.
    \end{align*}
    $\Rightarrow$: Assume that $\FI{X} = 0$. It follows from the proof of \Cref{feature importance: prop: range} that $\Dep{S \cup \{X\}}{Y} -  \Dep{S}{Y} \geq 0$ for every $S\subseteq \FISET \setminus \{X\}$. If $\Dep{S^* \cup \{X\}}{Y} -  \Dep{S^*}{Y} > 0$ for some given $S^* \in \FISET \setminus \{X\}$, it follows from \Cref{feature importance: eq: formal definition bp-fi} (with \Cref{feature importance: eq: definition w}) that
    \begin{align*}
        \FI{X} & =
        \sum_{S\subseteq \FISET \setminus \{X\}} w(S, \nvariables) \cdot \bigl ( \Dep{S\cup \{X\}}{Y} - \Dep{S}{Y} \bigr )
        \\[2ex]
               & \geq  w(S^*, \nvariables) \cdot \left ( \Dep{S^*\cup \{X\}}{Y} - \Dep{S^*}{Y} \right)                                       \\[2ex]
               & =  \frac{|S^*|! \cdot (\nvariables - |S^*| - 1)!}{\nvariables!} \cdot \left ( \Dep{S^*\cup \{X\}}{Y} - \Dep{S^*}{Y} \right) \\[2ex]
               & > 0.
    \end{align*}
    This gives a contradiction with the assumption that $\FI{X} = 0$, thus it is not possible that such an $S^*$ exists. This means that $\Dep{S\cup \{X\}}{Y} = \Dep{S}{Y}$ for all $S \in \FISET\setminus \{X\}$.
\end{proof}

\begin{relevance*}
    When a feature \emph{never} contributes any information, it is only fair that it does not receive any FI. The feature can be removed from the dataset, as it has no effect on the target variable. On the other hand, when a feature has an FI of zero, it would be unfair to this feature if it does in fact contribute information somewhere. It should then be rewarded some FI, albeit small it should be larger than zero.
\end{relevance*}

\paragraph*{Null-independence}
The property that a feature receives zero FI, when $\Dep{S\cup \{X\}}{Y} = \Dep{S}{Y}$ for all $S\in \FISET \setminus \{X\}$, is the same notion as a \emph{null player} in game theory. Berkelmans et al.~\cite{BP} show that $\Dep{X}{Y} = 0$, when $Y$ is \emph{independent} of $X$. To be a \emph{null player} requires a stricter definition of independence, which we call \emph{null-independence}. $Y$ is null-independent on $X$ if $\Dep{S\cup \{X\}}{Y} = \Dep{S}{Y}$ for all $S\in \FISET \setminus \{X\}$. In other words, $X$ is null-independent if and only if $\FI{X} = 0.$

\begin{corollary}
    Independent feature $\not \Rightarrow$ null-independent feature.
\end{corollary}

\begin{proof}
    Take e.g., the dataset consisting of two binary features $X_1, X_2\sim \mathcal{U}(\{0,1\})$ and a target variable $Y =  X_1 \cdot (1-X_2) + X_2 \cdot (1-X_1)$ which is the XOR of $X_1$ and $X_2$. Individually, the variables do not give any information about $Y$, whereas collectively they fully determine $Y$. In the proof of \Cref{feature importance: prop: xor dataset}, we show that this leads to $\FI{X_1} = \FI{X_2} = \frac{1}{2}$, whilst $\Dep{X_1}{Y} = \Dep{X_2}{Y} = 0$. Thus, $X_1$ and $X_2$ are \emph{independent}, but not \emph{null-independent}.
\end{proof}

\begin{corollary}
    Independent feature $\Leftarrow$ null-independent feature.
\end{corollary}

\begin{proof}
    When $X$ is \emph{null-independent}, it holds that $\FI{X} = 0$. Using \Cref{feature importance: prop: bounds}, we obtain \begin{align*}
        0 = \FI{X} \geq \frac{\Dep{X}{Y}}{\nvariables}  \Leftrightarrow \Dep{X}{Y} = 0.
    \end{align*}
    Thus, when $X$ is \emph{null-independent}, it is also \emph{independent}.
\end{proof}

\begin{corollary}
    Almost surely constant variables get zero FI.
\end{corollary}

\begin{proof}
    If $X$ is \emph{almost surely constant} (i.e., $\PP(X=x) = 1$), it immediately follows that ${\Dep{S\cup \{X\}}{Y} = \Dep{S}{Y}}$ for any $S\subseteq \FISET \setminus \{X\}$, as the distribution of $Y$ is not affected by $X$.
\end{proof}

\begin{property}[FI equal to one] \label{feature importance: prop: FI equal to one}
    When $\FI{X} = 1$, it holds that $\Dep{X}{Y} = 1$ and all other features are null-independent.
\end{property}

\begin{proof}
    As the \bpdep{} is bounded by $[0,1]$ \cite{BP}, it follows from \Cref{feature importance: prop: efficiency} that $\sum_{i \in \FISET} \FI{i} \leq 1$. Noting that each FI must be in $[0,1]$ due to \Cref{feature importance: prop: range}, we find that \begin{align*}
        \FI{X} = 1 \Rightarrow \FI{X'} = 0 \text{ for all } X' \in \FISET \setminus \{X\}.
    \end{align*}
    Thus all other features are \emph{null-independent}. Next, we show that $\Dep{X}{Y}=1$ must also hold, when $\FI{X} = 1$. Assume that $\Dep{X}{Y} < 1$. Using \Cref{feature importance: eq: formal definition bp-fi} (with \Cref{feature importance: eq: definition w,feature importance: eq: definition D}) we find that
    \begin{align*}
        1 & = \FI{X} = \sum_{S\subseteq \FISET \setminus \{X\}} w(S, \nvariables) \cdot D(X,Y,S)                                                                                                                 \\[2ex]
          & = \sum_{S\subseteq \FISET \setminus \{X\} : |S| > 0} \left ( w(S, \nvariables) \cdot D(X,Y,S) \right ) + w(\emptyset, \nvariables) \cdot D(X,Y,\emptyset)                                            \\[2ex]
          & \leq \sum_{S\subseteq \FISET \setminus \{X\} : |S| > 0} \left ( w(S, \nvariables) \cdot \left ( 1 - 0 \right)      \right )         + w(\emptyset, \nvariables) \cdot \left ( \Dep{X}{Y} - 0 \right) \\[2ex]
          & < \sum_{S\subseteq \FISET \setminus \{X\}} w(S, \nvariables)                                                                                                                                         \\[2ex]
          & = \sum_{k = 0}^{\nvariables - 1} \binom{\nvariables - 1}{k} \cdot \frac{k! \cdot (\nvariables - k - 1)!}{\nvariables!}                                                                               \\[2ex]
          & = \sum_{k = 0}^{\nvariables - 1} \frac{(\nvariables - 1)!}{k! \cdot (\nvariables - 1 -k)!} \cdot \frac{k! \cdot (\nvariables - k - 1)!}{\nvariables!}                                                \\[2ex]
          & = \sum_{k = 0}^{\nvariables - 1} \frac{1}{\nvariables}                                                                                                                                               \\[2ex]
          & = 1.
    \end{align*}
    Note that the inequality step follows from the range of the \bpdep{} (i.e., $[0,1]$). The largest possible addition is when $\Dep{S \cup \{X\}}{Y} - \Dep{S}{Y} = 1 - 0  =1$. This result gives a contradiction, as $1 < 1$ cannot be true, which means that $\Dep{X}{Y} = 1$.
\end{proof}

\begin{relevance*}
    When a variable gets an FI of one, the rest of the variables should be zero. Additionally, it should mean that this variable contains the necessary information to fully determine $Y$, which is why $\Dep{X}{Y} = 1$ should hold.
\end{relevance*}

\begin{property}
    $\Dep{X}{Y} = 1 \not \Rightarrow \FI{X} = 1.$
\end{property}

\begin{proof}
    As counterexample, examine the case where there are multiple variables that fully determine $Y$. \Cref{feature importance: prop: efficiency,feature importance: prop: range} must still hold. Thus, if FI is one for every variable that fully determines $Y$, we get \begin{align*}
        \sum_{i \in \FISET} \FI{i} \geq 1 + 1 \neq 1 = \Dep{\FISET}{Y},
    \end{align*}
    which is a contradiction.
\end{proof}

\begin{relevance*}
    This property is important for interpretation of the FI score. When $\FI{X} \neq 1$, it cannot be automatically concluded that $Y$ is not fully determined by $X$.
\end{relevance*}

If $Y$ is fully determined by $X$, we call $X$ \emph{fully informative}, as it gives all information that is necessary to determine $Y$.

\begin{property}[Max FI when fully informative] \label{feature importance: prop: max Fi when fully determined}
    If $X$ is fully informative, it holds that $\FI{i}\leq \FI{X}$ for any $i\in\FISET$.
\end{property}
\begin{proof}
    Assume that there exists a feature $i$ such that $\FI{i} > \FI{X}$, when $Y$ is fully determined by $X$. To attain a higher FI, somewhere in the sum of \Cref{feature importance: eq: formal definition bp-fi}, a higher gain must be made by $i$ compared to $X$. Observe that for any $S\subseteq \FISET \setminus\{i,X\}$ it holds that \begin{align*}
        \Dep{S\cup \{i\}}{Y} - \Dep{S}{Y} & \leq 1 - \Dep{S}{Y}                  \\[2ex]
                                          & = \Dep{S\cup\{X\}}{Y}  - \Dep{S}{Y}.
    \end{align*}
    For any $S\subseteq \FISET \setminus \{i\}$ with $X \in S$, it holds that
    \begin{align*}
        \Dep{S\cup \{i\}}{Y} - \Dep{S}{Y} & = \Dep{S\cup \{i\}}{Y} - 1 \\[2ex]
                                          & = 0.
    \end{align*}
    The last step follows from \Cref{feature importance: eq: dependency increasing}, as the dependency function is increasing, thus $\Dep{S\cup \{i\}}{Y} = 1$. In other words, no possible gain can be achieved with respect to $X$ in the Shapley values. Therefore, it cannot hold that $\FI{i} > \FI{X}$.
\end{proof}

\begin{relevance*}
    Whenever a variable fully determines $Y$, it should attain the highest FI. What would an FI higher than such a score mean? It gives more information than the maximal information? When this property would not hold, it would result in a confusing and difficult interpretation process.
\end{relevance*}

\begin{property}[Limiting the outcome space] \label{feature importance: prop: limiting outcome space}
    For any measurable function $f$ and RV $X$, replacing $X$ with $f(X)$ never increases the assigned FI to this variable.
\end{property}
\begin{proof}
    The \bpdep{} is non-increasing under functions of $X$ \cite{BP}. This means that for any measurable function $g$, it holds that \begin{align*}
        \Dep{g(X)}{Y} \leq \Dep{X}{Y}.
    \end{align*}
    Choose $g$ to be the function that maps the union of any feature set $S$ and the original RV $X$ to the union of $S$ and the replacement $f(X)$. In other words $g(S\cup \{X\}) = S\cup \{f(X)\}$ for any feature set $S$. It then follows that:
    \begin{align*}
        \Dep{S\cup \{f(X)\}}{Y} = \Dep{g(S \cup \{X\})}{Y} \leq \Dep{S \cup \{X\}}{Y},
    \end{align*}
    and \begin{align*}
        \Dep{S\cup \{f(X)\}}{Y} - \Dep{S}{Y} \leq \Dep{S \cup \{X\}}{Y} - \Dep{S}{Y}
    \end{align*}
    for any $S \subseteq \FISET \setminus \{X\}$. Thus, using \Cref{feature importance: eq: formal definition bp-fi}, we can conclude that replacing $X$ with $f(X)$ never increases the assigned FI.
\end{proof}

\begin{relevance*}
    This is an important observation for preprocessing. Whenever a variable is binned, it would receive less (or equal) FI when less bins are used. It could also potentially provide a useful upper bound, when the FI is already known before replacing $X$ with $f(X)$.
\end{relevance*}

\begin{corollary}
    For any measurable function $f$ and RV $X$, when $X = f(X')$ for another RV $X'$, replacing feature $X$ by feature $X'$ will never decrease the assigned FI.
\end{corollary}

\begin{proof}
    When $X = f(X')$ holds, it follows again (similar to \Cref{feature importance: prop: limiting outcome space}) that \begin{align*}
        \Dep{S\cup \{X\}}{Y} & = \Dep{S\cup \{f(X')\}}{Y} \leq \Dep{S \cup \{X'\}}{Y}
    \end{align*}
    for any $S \subseteq \FISET \setminus \{X\}$. Therefore, using \Cref{feature importance: eq: formal definition bp-fi}, observe that replacing $X$ with $X'$ never decreases the assigned FI.
\end{proof}

Shapley values have additional properties when the characteristic function $v$ is \emph{subadditive} and/or \emph{superadditive} \cite{Roth}. We show that our function, defined by \Cref{feature importance: eq: value function bp-fi}, is neither.

\begin{property}[Neither subadditive nor superadditive] \label{feature importance: prop: not subadditive or superadditive}
    Our characteristic function $v(S)= \Dep{S}{Y}$ is neither \emph{subadditive} nor \emph{superadditive}.
\end{property}

\begin{proof}
    Consider the following two counterexamples.

    \emph{Counterexample subadditive:}
    A function $f$ is \emph{subadditive} if for any $S,T\in \FISET$ it holds that \begin{align*}
        f(S\cup T) \leq f(S) + f(T).
    \end{align*}
    Examine the dataset consisting of two binary features $X_1, X_2\sim \mathcal{U}(\{0,1\})$ and a target variable $Y =  X_1 \cdot (1-X_2) + X_2 \cdot (1-X_1)$ which is the XOR of $X_1$ and $X_2$. Both $X_1$ and $X_2$ do not individually give any new information about the distribution of $Y$, thus $v(X_1) = v(X_2) = 0$ (see properties of the \bpdep{} \cite{BP}). However, collectively they fully determine $Y$ and thus $v(X_1 \cup X_2) = 1$. We can therefore conclude that $v$ is not subadditive, as \begin{align*}
        v(X_1 \cup X_2) = 1 \not \leq 0 + 0 = v(X_1) + v(X_2).
    \end{align*}

    \emph{Counterexample superadditive:}
    A function $f$ is \emph{superadditive} if for any $S,T\in \FISET$ it holds that \begin{align*}
        f(S\cup T) \geq f(S) + f(T).
    \end{align*}
    Consider the dataset consisting of two binary features $X \sim \mathcal{U}(\{0,1\})$ and a \emph{clone} $X^{\text{clone}} := X$, where the target variable $Y$ is defined as $Y:= X.$ Note that both $X$ and $X^{\text{clone}}$ fully determine $Y$, thus $v(X) = v(X^{\text{clone}}) = 1$ (see properties of the \bpdep{} \cite{BP}). Combining $X$ and $X^{\text{clone}}$ also fully determines $Y$, which leads to: \begin{align*}
        v(X \cup X^{\text{clone}}) = 1 \not \geq 1 + 1 = v(X) + v(X^{\text{clone}}).
    \end{align*}
    Thus, $v$ is also not superadditive.
\end{proof}

\begin{relevance*}
    If the characteristic function $v$ is \emph{subadditive}, it would hold that $\FI{X} \leq v(X)$ for any $X\in\FISET$. When $v$ is \emph{superadditive}, it follows that $\FI{X} \geq v(X)$ for any $X\in \FISET$. This is sometimes also referred to as \emph{individual rationality}, which means that no player receives less, than what he could get on his own. This makes sense in a game-theoretic scenario with human players that can decide to not play when one could gain more by not cooperating. In our case, features do not have a free will, which makes this property not necessary. The above proof shows that $v$ is in our case neither \emph{subadditive} nor \emph{superadditive}, which is why we cannot use their corresponding bounds.
\end{relevance*}

\begin{property}[Adding features can increase FI] \label{feature importance: prop: adding features can increase FI}
    When an extra feature is added to the dataset, the FI of $X$ can increase.
\end{property}
\begin{proof}
    Consider the previously mentioned XOR dataset, where $X_1, X_2 \sim \mathcal{U}(\{0,1\})$ and $Y =  X_1 \cdot (1-X_2) + X_2 \cdot (1-X_1)$. If at first, $X_2$ was not in the dataset, the FI of $X_1$ would be zero, as $\Dep{X_1}{Y} = 0$. However, if $X_2$ is added to the dataset, the FI of $X_1$ increases to $\frac{1}{2}$ (see \Cref{feature importance: prop: xor dataset}). The FI of a feature can thus increase if another feature is added.
\end{proof}

\begin{property}[Adding features can decrease FI] \label{feature importance: prop: adding features can decrease FI}
    When an extra feature is added to the dataset, the FI of $X$ can decrease.
\end{property}

\begin{proof}
    Consider the dataset given by $X \sim \mathcal{U}(\{0,1\})$ and $Y:= X.$ It immediately follows that $\FI{X} = 1$. However, when a \emph{clone} is introduced ($X^{\text{clone}}:=X$), it holds that $\FI{X} = \FI{X^{\text{clone}}}$, because of \Cref{feature importance: prop: max Fi when fully determined}. Additionally, it follows from \Cref{feature importance: prop: efficiency} that $\FI{X} + \FI{X^{\text{clone}}} = 1$. Thus, $\FI{X} = \frac{1}{2}$, and the FI of a variable can therefore be decreased if another variable is added.
\end{proof}

\begin{relevance*}
    It is important to observe that the FI of a variable is dependent on the other features (\Cref{feature importance: prop: adding features can increase FI,feature importance: prop: adding features can decrease FI}). Adding or removing features could change the FI, which one needs to be aware of.
\end{relevance*}

\begin{property}[Cloning does not increase FI] \label{feature importance: prop: cloning does not increase FI}
    For any RV $X\in \FISET$, adding an identical variable $X^{\text{clone}}:= X$ (cloning) to the dataset, does not increase the FI of $X$.
\end{property}

\begin{proof}
    Let $\FINEW{X}$ denote the FI of $X$ after the clone $X^{\text{clone}}$ is added. Using \Cref{feature importance: eq: formal definition bp-fi} (with \Cref{feature importance: eq: definition w,feature importance: eq: definition D}), we
    find
    \begin{align*}
        \FINEW{X} & =
        \sum_{S\subseteq \FISET \cup \{X^{\text{clone}}\} \setminus \{X\}} w(S, \nvariables + 1) \cdot D(X,Y,S)                                                                                                      \\[4ex]
                  & \stackrel{(a)}{=}
        \begin{multlined}[t][0.76\linewidth]
            \sum_{S\subseteq \FISET \cup \{X^{\text{clone}}\} \setminus \{X\} : X^{\text{clone}} \in S} w(S, \nvariables + 1) \cdot
            D(X,Y,S) \\[1ex]
            + \sum_{S\subseteq \FISET \cup \{X^{\text{clone}}\} \setminus \{X\} : X^{\text{clone}} \not \in S} w(S, \nvariables + 1) \cdot D(X,Y,S)
        \end{multlined}                                       \\[4ex]
                  & \stackrel{(b)}{=}
        \begin{multlined}[t][0.76\linewidth]
            \sum_{S\subseteq \FISET \cup \{X^{\text{clone}}\} \setminus \{X\} : X^{\text{clone}} \in S} w(S, \nvariables + 1) \cdot 0                                                 \\[1ex]
            + \sum_{S\subseteq \FISET \cup \{X^{\text{clone}}\} \setminus \{X\} : X^{\text{clone}} \not \in S} w(S, \nvariables + 1) \cdot D(X,Y,S)
        \end{multlined} \\[4ex]
                  & =
        \sum_{S\subseteq \FISET \setminus \{X\}} w(S, \nvariables + 1) \cdot  D(X,Y,S).
    \end{align*}
    Equality (a) follows by splitting the sum over all subsets of $\FISET \cup \{X^{\text{clone}}\} \setminus \{X\}$ whether $X^{\text{clone}}$ is part of the subset or not. Adding $X$ to a subset that already contains the clone $X^{\text{clone}}$ does not change the \bpdep, which is why Equality (b) follows. The takeaway from this derivation is that the sum over all subsets $S\subseteq \FISET \cup \{X^{\text{clone}}\} \setminus \{X\}$ reduces to the sum over $S\subseteq \FISET \setminus \{X\}$.

    Comparing the new $\FINEW{X}$ with the original $\FI{X}$ gives
    \begin{multline*}
        \FI{X} - \FINEW{X}  =  \sum_{S\subseteq \FISET \setminus \{X\}} w(S, \nvariables) \cdot D(X,Y,S) \\[1ex]
        - \sum_{S\subseteq \FISET \setminus \{X\}} w(S, \nvariables + 1) \cdot  D(X,Y,S).
    \end{multline*}
    Using \Cref{feature importance: eq: definition w}, we find that
    \begin{align*}
        \frac{w(S, \nvariables + 1)}{w(S, \nvariables)} & = \frac{\frac{|S|! \cdot (\nvariables + 1 - |S| - 1)!}{(\nvariables + 1)!}}{\frac{|S|! \cdot (\nvariables - |S| - 1)!}{\nvariables!}} = \frac{\nvariables - \vert S \vert}{\nvariables + 1} < 1,
    \end{align*}
    thus $\FI{X} - \FINEW{X}\geq 0$ with equality if and only if $\FI{X} = 0$. Therefore, we can conclude that cloning a variable cannot increase the FI of $X$ and will decrease the FI when $X$ is \emph{null-independent}.
\end{proof}

\begin{relevance*}
    We consider this a natural property of a good FI method, as no logical reason can be found why adding the exact same information would lead to an increase in FI for the original variable. The information a variable contains only becomes less valuable, as it becomes common knowledge.
\end{relevance*}

\begin{property}[Order does not change FI] \label{feature importance: prop: order does not change FI}
    The order of the features does not affect the individually assigned FI. Consider the datasets $[X_1, X_2, \dots, X_{\nvariables}]$ and $[Z_1, Z_2, \dots, Z_{\nvariables}]$, where $Z_{\pi(i)} = X_i$ for some permutation $\pi$. It holds that $\FI{X_i} = \FI{Z_{\pi(i)}}$ for any $i\in \{1,\dots, \nvariables\}.$
\end{property}
\begin{proof}
    Note that the order of features nowhere plays a roll in the definition of \bpfi{} (\Cref{feature importance: eq: formal definition bp-fi}). The \bpdep{} is also independent of the given order, which is why this property trivially holds.
\end{proof}

\begin{relevance*}
    This is a very natural property of a good FI. Consider what would happen if the FI is dependent on the order in the dataset. Should all possible orders be evaluated and averaged to receive a final FI? We cannot find any arguments why someone should want FI to be dependent on the order of features.
\end{relevance*}

\subsection*{Datasets}
Next, we consider a few datasets, where we derive the theoretical outcome for the \bpfi{}. These datasets are also used in \Cref{feature importance: subsec: Property evaluation} to test FI methods. It is very hard to evaluate FI methods, as the ground truth is often unknown. However, we believe that the FI outcomes on these datasets are all natural and defendable. However, it remains subjective what one considers to be the `correct' FI values.

\begin{property}[XOR dataset] \label{feature importance: prop: xor dataset}
    Consider the following dataset consisting of two binary features $X_1, X_2\sim \mathcal{U}(\{0,1\})$ and a target variable $Y =  X_1 \cdot (1-X_2) + X_2 \cdot (1-X_1)$ which is the XOR of $X_1$ and $X_2$. It holds that \begin{align*}
        \FI{X_1} = \FI{X_2} = \frac{1}{2}.
    \end{align*}
\end{property}
\begin{proof}
    Observe that $\Dep{X_1}{Y} = \Dep{X_2}{Y} = 0$ and $\Dep{X_1 \cup X_2}{Y} = 1.$ With \Cref{feature importance: eq: formal definition bp-fi}, it follows that \begin{align*}
        \FI{X_1} & =  \sum_{S\subseteq \{1,2\} \setminus X_1} \frac{|S|! \cdot (1 - |S|)!}{2!} \cdot \left ( \Dep{S\cup X_1}{Y} - \Dep{S}{Y} \right)        \\[4ex]
                 & =
        \begin{multlined}[t][0.76\linewidth]
            \frac{|\{\emptyset\}|! \cdot (1 - |\{\emptyset\}|)!}{2!} \cdot \left ( \Dep{\{\emptyset\}\cup X_1}{Y} - \Dep{\{\emptyset\}}{Y} \right)       \\[1ex]
            + \frac{|\{X_2\}|! \cdot (1 - |\{X_2\}|)!}{2!} \cdot \left ( \Dep{X_1\cup X_2}{Y} - \Dep{X_2}{Y} \right)
        \end{multlined} \\[4ex]
                 & =
        \frac{1}{2} \cdot \left ( \Dep{X_1}{Y} - 0  \right ) + \frac{1}{2} \cdot \left ( \Dep{X_1 \cup X_2}{Y} - \Dep{X_2}{Y} \right)                       \\[2ex]
                 & =
        \frac{1}{2} \cdot 0 + \frac{1}{2} \cdot \left ( 1 - 0 \right)                                                                                       \\[2ex]
                 & =  \frac{1}{2}.
    \end{align*}
    Using \Cref{feature importance: prop: efficiency}, it follows that $\FI{X_2} = 1 - \FI{X_1} = \frac{1}{2}$.
\end{proof}

\begin{relevance*}
    This XOR formula is discussed and used to test FI methods in \cite{Fryer2021}. However, they only test for equality ($\FI{X_1} = \FI{X_2}$), not the specific value. Due to \emph{symmetry}, we would also argue that both $X_1$ and $X_2$ should get the same FI, as they fulfill the same role. Together, they fully determine $Y$, which is why the total FI should be one (see \Cref{feature importance: prop: FI equal to one}). Dividing this equally amongst the two variables, gives a logical desirable FI outcome of $\frac{1}{2}$ for each variable.
\end{relevance*}

\begin{property}[Probability dataset] \label{feature importance: prop: probability dataset}
    Consider the following dataset consisting of $Y = \lfloor X_S / 2 \rfloor$ and $X_i = Z_i + (S-1)$ with $Z_i \sim\UU{\{0,2\}}$ for $i = 1,2$ and $\PP(S=1) = p$, $\PP(S=2)= 1-p$. It holds that \begin{align*}
        \FI{X_1} = p \text{ and } \FI{X_2} = 1-p.
    \end{align*}
\end{property}

\begin{proof}

    Observe that by \Cref{feature importance: eq: definition discrete ud} \begin{align*}
        \UD{X_1}{Y} & =\sum_{x_1 \in \{0,1,2,3\}} p_{X_1}(x_1) \cdot \sum_{y \in \{0,1\}} \left \vert p_{Y\vert X_1=x_1}(y)   - p_Y(y) \right\vert              \\[4ex]
                    & =
        \begin{multlined}[t][0.76\linewidth]
            \sum_{x_1 \in \{0,2\}} p_{X_1}(x_1) \cdot \sum_{y \in \{0,1\}} \left \vert p_{Y\vert X_1=x_1}(y)   - \frac{1}{2} \right\vert                 \\[1ex]
            +  \sum_{x_1 \in \{1,3\}} p_{X_1}(x_1) \cdot \sum_{y \in \{0,1\}} \left \vert p_{Y\vert X_1=x_1}(y)   - \frac{1}{2} \right\vert         \end{multlined} \\[4ex]
                    & =
        \begin{multlined}[t][0.76\linewidth]
            \sum_{x_1 \in \{0,2\}} \frac{p}{2} \cdot \left (\left \vert 1 - \frac{1}{2} \right \vert + \left \vert 0 - \frac{1}{2} \right \vert \right ) \\[1ex]
            +  \sum_{x_1 \in \{1,3\}} \frac{1-p}{2} \cdot \sum_{y \in \{0,1\}} \left \vert p_{Y}(y)   - p_Y(y) \right\vert                          \end{multlined} \\[4ex]
                    & = p.
    \end{align*}
    Similarly, it follows that $\UD{X_2}{Y} = 1-p$.
    \begin{align*}
        \UD{Y}{Y} & =\sum_{y' \in \{0,1\}} p_{Y}(y') \cdot \sum_{y \in \{0,1\}} \left \vert p_{Y\vert Y=y'}(y)   - p_Y(y) \right\vert                             \\[2ex]
                  & = \sum_{y' \in \{0,1\}} \frac{1}{2} \cdot \left (\left \vert 1 - \frac{1}{2} \right \vert + \left \vert 0 - \frac{1}{2} \right \vert \right ) \\[2ex]
                  & = 1.
    \end{align*}

    From \Cref{feature importance: eq: definition bp dependency}, it follows that $\Dep{X_1}{Y} = p$ and $\Dep{X_2}{Y} = 1-p$. Additionally, note that knowing $X_1$ and $X_2$ fully determines $Y$, thus $\Dep{X_1\cup X_2}{Y} = 1$. With \Cref{feature importance: eq: formal definition bp-fi}, we now find
    \begin{align*}
        \FI{X_1} & = \sum_{S\subseteq \{X_1,X_2\} \setminus X_1} \frac{|S|! \cdot (1 - |S|)!}{2!} \cdot \left ( \Dep{S\cup X_1}{Y} - \Dep{S}{Y} \right)        \\[4ex]
                 & =
        \begin{multlined}[t][0.76\linewidth]
            \frac{|\{\emptyset\}|! \cdot (1 - |\{\emptyset\}|)!}{2!} \cdot \left ( \Dep{\{\emptyset\}\cup X_1}{Y} - \Dep{\{\emptyset\}}{Y} \right)          \\[1ex]
            + \frac{|\{X_2\}|! \cdot (1 - |\{X_2\}|)!}{2!} \cdot \left ( \Dep{X_1\cup X_2}{Y} - \Dep{X_2}{Y} \right)                 \end{multlined} \\[4ex]
                 & =
        \frac{1}{2} \cdot \left ( \Dep{X_1}{Y} - 0  \right ) + \frac{1}{2} \cdot \left ( \Dep{X_1 \cup X_2}{Y} - \Dep{X_2}{Y} \right)                          \\[2ex]
                 & =           \frac{1}{2} \cdot \left ( p - 0  \right ) + \frac{1}{2} \cdot \left ( 1 - (1-p) \right)                                         \\[2ex]
                 & = \frac{p}{2} + \frac{p}{2} = p.
    \end{align*}
    Using \Cref{feature importance: prop: efficiency}, it follows that $\FI{X_2} = 1 - \FI{X_1} = 1-p.$
\end{proof}

\begin{relevance*}
    At first glance, it is not immediately clear why these FI values are natural, which is why we discuss this dataset in more detail. $S$ can be considered a selection parameter that determines if $X_1$ or $X_2$ is used for $Y$ with probability $p$ and $1-p$, respectively. $X_i$ is constructed in such a way that it is uniformly drawn from $\{0,2\}$ or $\{1,3\}$ depending on $S$. However, as $Y = \lfloor X_S / 2 \rfloor$, it holds that $X_S=0$ and $X_S=1$ give the same outcome for $Y$. The same holds for $X_S= 2$ and $X_S =3$. Therefore, note that the distribution of $Y$ is independent of the selection parameter $S$. Knowing $X_1$ gives the following information. First, $S$ can be derived from the value of $X_1.$ When $X_1\in\{0,2\}$ it must hold that $S=1$, and if $X_1\in\{1,3\}$ it follows that $S=2$. Second, when $S=1$ it means that $Y$ is fully determined by $X_1$. If $S=2$, knowing that $X_1 = 1$ or $X_1 = 3$ does not provide any additional information about $Y$. With probability $p$ knowing $X_1$ will fully determine $Y$, whereas with probability $1-p$, it will provide no information about the distribution of $Y$. The outcome $\FI{X_1} = p$, is therefore very natural. The same argumentation applies for $X_2$, which leads to $\FI{X_2} = 1-p.$

\end{relevance*}

\section{Comparing with existing methods} \label{feature importance: sec: Comparing with existing methods}


In the previous section, we showed that \bpfi{} has many desirable properties. Next, we evaluate for a large collection of FI methods if the properties hold for several synthetic datasets. Note that these datasets can only be used as counterexample, not as proof of a property. First, we discuss the in \Cref{feature importance: subsec: Alternative FI methods} the FI methods that are investigated. Second, we give the datasets (\Cref{feature importance: subsec: Synthetic datasets}) and explain how they are used to test the properties (\Cref{feature importance: subsec: Property evaluation}). The results are discussed in \Cref{feature importance: subsec: Evaluation results}.

\subsection{Alternative FI methods} \label{feature importance: subsec: Alternative FI methods}


A wide range of FI methods have been suggested for all kinds of situations. It is therefore first necessary to discuss the major categorical differences between them.

\paragraph*{Global vs. local}
An important distinction to make for FI methods is whether they are constructed for \emph{local} or \emph{global} explanations. \emph{Global} FI methods give an importance score for each feature over the entire dataset, whereas \emph{local} FI methods explain which variables were important for a single example \cite{Ghorbani2022}. The global and local scores do not have to coincide: \enquote{features that are globally important may not be important in the local context, and vice versa} \cite{Ribeiro2016}. This research is focussed on global FI methods, but sometimes a local FI approach can be averaged out to obtain a global FI. For example, in \cite{Lundberg2018} a local FI method is introduced called \emph{Tree SHAP}. It is also used globally, by averaging the absolute values of the local FI.

\paragraph*{Model-specific vs. model-agnostic}
A distinction within FI methods can be made between model-\emph{specific} and -\emph{agnostic} methods. \emph{Model-specific} methods aim to find the FI using a prediction model such as a neural network or random forest, whereas \emph{model-agnostic} methods do not use a prediction model. The \bpfi{} is model-agnostic, which therefore gives insights into the dataset. Whenever a model-specific method is used, the focus lies more on gaining information about the prediction model, not the dataset. In our tests, we use both model-specific and -agnostic methods.

\paragraph*{Classification vs. regression}
Depending on the exact dataset, the target variable is either \emph{categorical} or \emph{numerical}, which is precisely the difference between \emph{classification} and \emph{regression}. Not all existing FI methods can handle both cases. In this research, we generate synthetic \emph{classification} datasets, so we only examine FI methods that are intended for these cases. An additional problem with regression datasets, is that continuous variables need to be converted to discrete bins. This conversion could drastically change the FI scores, which makes it harder to draw fair conclusions.

\paragraph*{Collection}
We have gathered the largest known collection of FI methods from various sources \cite{scikit-learn, Rvip, Rcaret, RFSinR, Shap, Onig2021, XGBoost, Sage, qii, Scipy, Rebelosa, Relief, ITMO, Fryer2021, Abe2005, Rinfotheo, Fisher, Ghorbani2022, RPartykit, RVita, Treeinterpreter, Diffi} or implemented them ourselves. This has been done with the following policy: Whenever code of a \emph{classification} FI method was available in \textsf{R} or Python or the implementation was relatively straightforward, it was added to the collection. This resulted in 196 \emph{base} methods and \var{n_methods} total methods, as some base methods can be combined with \emph{multiple} machine learning approaches or selection objectives, see \Cref{feature importance: tab: FI methods}. However, beware that most methods also contain additional parameters, which are not investigated in this research. The \emph{default} values for these parameters are always used.

\begingroup

\begin{table*}
    \fboxsep=2mm \fboxrule=0.1mm
    \caption{\textbf{All evaluated FI methods:} List of all FI methods that are evaluated in the experiments. The colored methods work in combination with multiple options: {\footnotesize \slshape
    Logistic Regression\textsuperscript{\cone, \ctwo, \cthree},
    Ridge\textsuperscript{\cone, \ctwo},
    Linear Regression\textsuperscript{\cone, \ctwo},
    Lasso\textsuperscript{\cone, \ctwo},
    SGD Classifier\textsuperscript{\cone, \cthree},
    MLP Classifier\textsuperscript{\cone, \ctwo},
    K Neighbors Classifier\textsuperscript{\cone, \ctwo},
    Gradient Boosting Classifier\textsuperscript{\cone, \ctwo, \cfour},
    AdaBoost Classifier\textsuperscript{\cone, \ctwo},
    Gaussian NB\textsuperscript{\cone, \ctwo},
    Bernoulli NB\textsuperscript{\cone, \ctwo},
    Linear Discriminant Analysis\textsuperscript{\cone, \ctwo},
    Decision Tree Classifier\textsuperscript{\cone, \ctwo, \cfour, \cfive},
    Random Forest Classifier\textsuperscript{\cone, \ctwo, \cfour, \cfive},
    SVC\textsuperscript{\cone},
    CatBoost Classifier\textsuperscript{\cone, \ctwo},
    LGBM Classifier\textsuperscript{\cone, \ctwo, \cfour},
    XGB Classifier\textsuperscript{\cone, \ctwo, \cfour, \cseven},
    XGBRF Classifier\textsuperscript{\cone, \ctwo, \cfour, \cseven},
    ExtraTree Classifier\textsuperscript{\cfour, \cfive},
    ExtraTrees Classifier\textsuperscript{\cfour, \cfive},
    plsda\textsuperscript{\csix},
    splsda\textsuperscript{\csix},
    gini\textsuperscript{\ceight},
    entropy\textsuperscript{\ceight},
    NN1\textsuperscript{\cnine},
    NN2\textsuperscript{\cnine}}.
    This leads to a total of 468 FI methods from various sources \cite{scikit-learn, Rvip, Rcaret, RFSinR, Shap, Onig2021, XGBoost, Sage, qii, Scipy, Rebelosa, Relief, ITMO, Fryer2021, Abe2005, Rinfotheo, Fisher, Ghorbani2022, RPartykit, RVita, Treeinterpreter, Diffi} or self-implemented.
    }

    \label{feature importance: tab: FI methods}
    \begin{center}
        \resizebox*{1.0\linewidth}{!}{
            \begin{tikzpicture}
                \def\colone{0.0};
                \def\collone{5.2};
                \def\coltwo{5.2};
                \def\colltwo{10.5};
                \def\colthree{0.2};
                \def\collthree{17.0};
                \def\colfour{17.0};
                \def\collfour{22.24};
                \def\rowh{0.422};
                \def\starty{20.85};

                \def\addx{1};

                \draw[fill = scheme4] (0,49 * \rowh+ + 0.17) rectangle (\collfour, 51 *\rowh + 0.06);

                \definecolor{colorborder1}{RGB}{72,191,142}
                \colorlet{colorback1}{colorborder1!20!white}
                \def\yh{\starty};
                \draw[colorborder1, dashed, fill = colorback1] (\colone, \yh) -- (\collfour, \yh) -- (\collfour, \yh- 3*\rowh) -- (\colone, \yh- 3*\rowh) -- cycle;


                \definecolor{colorborder2}{RGB}{153,28,100}
                \colorlet{colorback2}{colorborder2!20!white}
                \def\yh{\starty - 3 * \rowh};
                \draw[colorborder2, dashed, fill = colorback2] (\colone, \yh) -- (\collfour, \yh) -- (\collfour, \yh- 2*\rowh) -- (\colone, \yh- 2*\rowh) -- cycle;


                \definecolor{colorborder3}{RGB}{63,227,75}
                \colorlet{colorback3}{colorborder3!20!white}
                \def\yh{\starty - 5 * \rowh};
                \draw[colorborder3, dashed, fill = colorback3] (\colone, \yh) -- (\collfour, \yh) -- (\collfour, \yh- 1*\rowh) -- (\collthree, \yh- 1*\rowh) -- (\collthree, \yh- 2*\rowh) -- (\colone, \yh- 2*\rowh) -- cycle;


                \definecolor{colorborder4}{RGB}{163,53,200}
                \colorlet{colorback4}{colorborder4!20!white}
                \def\yh{\starty - 6 * \rowh};
                \draw[colorborder4, dashed, fill = colorback4] (\collfour, \yh) -- (\colfour, \yh) -- (\colfour, \yh- 1*\rowh) -- (\collfour, \yh- 1*\rowh);
                \def\yh{\starty - 7 * \rowh};
                \draw[colorborder4, dashed, fill = colorback4] (\colone, \yh) -- (\collone, \yh) -- (\collone, \yh- 1*\rowh) -- (\colone, \yh- 1*\rowh);


                \definecolor{colorborder5}{RGB}{186,227,66}
                \colorlet{colorback5}{colorborder5!20!white}
                \def\yh{\starty - 7 * \rowh};
                \draw[colorborder5, dashed, fill = colorback5] (\coltwo, \yh) -- (\collfour, \yh) -- (\collfour, \yh-1*\rowh) -- (\coltwo, \yh-1*\rowh) -- cycle;


                \definecolor{colorborder6}{RGB}{253,181,172}
                \colorlet{colorback6}{colorborder6!20!white}
                \def\yh{\starty - 8 * \rowh};
                \draw[colorborder6, dashed, fill = colorback6] (\colone, \yh) -- (\collfour, \yh) -- (\collfour, \yh- 3*\rowh) -- (\colone, \yh- 3*\rowh) -- cycle;


                \definecolor{colorborder7}{RGB}{161,207,207}
                \colorlet{colorback7}{colorborder7!20!white}
                \def\yh{\starty - 11 * \rowh};
                \draw[colorborder7, dashed, fill = colorback7] (\colone, \yh) -- (\collthree, \yh) -- (\collthree, \yh- 1*\rowh) -- (\colone, \yh- 1*\rowh) -- cycle;


                \definecolor{colorborder8}{RGB}{144,14,8}
                \colorlet{colorback8}{colorborder8!20!white}
                \def\yh{\starty - 11 * \rowh};
                \draw[colorborder8, dashed, fill = colorback8] (\colfour, \yh) -- (\collfour, \yh) -- (\collfour, \yh- 16*\rowh) -- (\collone, \yh- 16*\rowh) --  (\collone, \yh- 17*\rowh) -- (\colone, \yh- 17*\rowh) -- (\colone, \yh- 1*\rowh) -- (\colfour, \yh- 1*\rowh) -- cycle;


                \definecolor{colorborder9}{RGB}{230,195,82}
                \colorlet{colorback9}{colorborder9!20!white}
                \def\yh{\starty - 27 * \rowh};
                \draw[colorborder9, dashed, fill = colorback9] (\coltwo, \yh) -- (\collfour, \yh) -- (\collfour, \yh- 9*\rowh) -- (\colone, \yh- 9*\rowh) -- (\colone, \yh- 1*\rowh) -- (\coltwo, \yh- 1*\rowh) -- cycle;


                \colorlet{colorborder10}{red!10!orange}
                \colorlet{colorback10}{colorborder10!20!white}
                \def\yh{\starty - 36 * \rowh};
                \draw[colorborder10, dashed, fill = colorback10] (\colone, \yh) -- (\collfour, \yh) -- (\collfour, \yh- 2*\rowh) -- (\collthree, \yh- 2*\rowh) -- (\collthree, \yh- 3*\rowh) -- (\colone, \yh- 3*\rowh) -- cycle;


                \definecolor{colorborder11}{RGB}{101,161,14}
                \colorlet{colorback11}{colorborder11!20!white}
                \def\yh{\starty - 38 * \rowh};
                \draw[colorborder11, dashed, fill = colorback11] (\collfour, \yh) -- (\colfour, \yh) -- (\colfour, \yh- 1*\rowh) -- (\collfour, \yh- 1*\rowh);
                \def\yh{\starty - 39 * \rowh};
                \draw[colorborder11, dashed, fill = colorback11] (\colone, \yh) -- (\collthree, \yh) -- (\collthree, \yh- 1*\rowh) -- (\colone, \yh- 1*\rowh);


                \definecolor{colorborder12}{RGB}{8,87,130}
                \colorlet{colorback12}{colorborder12!20!white}
                \def\yh{\starty - 39 * \rowh};
                \draw[colorborder12, dashed, fill = colorback12] (\collfour, \yh) -- (\colfour, \yh) -- (\colfour, \yh- 1*\rowh) -- (\collfour, \yh- 1*\rowh);
                \def\yh{\starty - 40 * \rowh};
                \draw[colorborder12, dashed, fill = colorback12] (\colone, \yh) -- (\collone, \yh) -- (\collone, \yh- 1*\rowh) -- (\colone, \yh- 1*\rowh);


                \definecolor{colorborder13}{RGB}{253,146,250}
                \colorlet{colorback13}{colorborder13!20!white}
                \def\yh{\starty - 40 * \rowh};
                \draw[colorborder13, dashed, fill = colorback13] (\collfour, \yh) -- (\coltwo, \yh) -- (\coltwo, \yh- 1*\rowh) -- (\collfour, \yh- 1*\rowh);
                \def\yh{\starty - 41 * \rowh};
                \draw[colorborder13, dashed, fill = colorback13] (\colone, \yh) -- (\collone, \yh) -- (\collone, \yh- 1*\rowh) -- (\colone, \yh- 1*\rowh);


                \definecolor{colorborder14}{RGB}{104,60,0}
                \colorlet{colorback14}{colorborder14!20!white}
                \def\yh{\starty - 41 * \rowh};
                \draw[colorborder14, dashed, fill = colorback14] (\coltwo, \yh) -- (\collfour, \yh) -- (\collfour, \yh-\rowh) -- (\coltwo, \yh-\rowh) -- cycle;


                \definecolor{colorborder15}{RGB}{36,165,247}
                \colorlet{colorback15}{colorborder15!20!white}
                \def\yh{\starty - 42 * \rowh};
                \draw[colorborder15, dashed, fill = colorback15] (\colone, \yh) -- (\collfour, \yh) -- (\collfour, \yh- 7*\rowh) -- (\colone, \yh- 7*\rowh) -- cycle;


                \definecolor{colorborder16}{RGB}{180,114,107}
                \colorlet{colorback16}{colorborder16!20!white}
                \def\yh{\starty - 49 * \rowh};
                \draw[colorborder16, dashed, fill = colorback16] (\colone, \yh) -- (\collfour, \yh) -- (\collfour, \yh- 1*\rowh) -- (\collone, \yh- 1*\rowh) -- (\collone, \yh- 2*\rowh - 0.165) -- (\colone, \yh- 2*\rowh - 0.165) -- cycle;


                \definecolor{colorborder17}{RGB}{251,9,152}
                \colorlet{colorback17}{colorborder17!20!white}
                \def\yh{\starty - 48 * \rowh};
                \draw[colorborder17, dashed, fill = colorback17] (\coltwo, \yh - 2 * \rowh) -- (\colltwo, \yh - 2 * \rowh) -- (\colltwo, \yh- 3*\rowh - 0.165) -- (\coltwo, \yh- 3*\rowh - 0.165) -- cycle;


                \tikzset{shift = {(0,-2*\rowh)}}
                \node[rectangle, draw = boxcolor, align = center, anchor = south west] at (0,0){
                    {\headfont \bf Feature Importance methods} \\[0.5em]  \footnotesize
                    \begin{tabular}{llll}
                        1. AdaBoost Classifier                   & 2. \itemcoloreight{Random Forest Classifier} & 3. \itemcoloreight{Extra Trees Classifier}          & 4. Gradient Boosting Classifier          \\
                        5. SVR absolute weights                  & 6. EL absolute weights                       & 7. \itemcolorone{Permutation Importance Classifier} & 8. PCA sum                               \\
                        9. PCA weighted                          & 10. chi2                                     & 11. f classif                                       & 12. mutual info classif                  \\
                        13. KL divergence                        & 14. \textsf{R} Mutual Information            & 15. Fisher Score                                    & 16. FeatureVec                           \\
                        17. \textsf{R} Varimp Classifier         & 18. \textsf{R} PIMP Classifier               & 19. \itemcolorfive{Treeinterpreter Classifier}      & 20. DIFFI                                \\
                        21. \itemcolorfour{Tree Classifier}      & 22. \itemcolorthree{Linear Classifier}       & 23. \itemcolorone{Permutation Classifier}           & 24. \itemcolorone{Partition Classifier}  \\
                        25. \itemcolorone{Sampling Classifier}   & 26. \itemcolorone{Kernel Classifier}         & 27. \itemcolorone{Exact Classifier}                 & 28. \itemcolorone{RFI Classifier}        \\
                        29. \itemcolorone{CFI Classifier}        & 30. \itemcolorsix{Sum Classifier}            & 31. \itemcolorsix{Weighted X Classifier}            & 32. \itemcolorsix{Weighted Y Classifier} \\
                        33. f oneway                             & 34. alexandergovern                          & 35. pearsonr                                        & 36. spearmanr                            \\
                        37. pointbiserialr                       & 38. kendalltau                               & 39. weightedtau                                     & 40. somersd                              \\
                        41. linregress                           & 42. siegelslopes                             & 43. theilslopes                                     & 44. multiscale graphcorr                 \\
                        45. \itemcolorseven{booster weight}      & 46. \itemcolorseven{booster gain}            & 47. \itemcolorseven{booster cover}                  & 48. snn                                  \\
                        49. knn                                  & 50. bayesglm                                 & 51. lssvmRadial                                     & 52. rocc                                 \\
                        53. ownn                                 & 54. ORFpls                                   & 55. rFerns                                          & 56. treebag                              \\
                        57. RRF                                  & 58. svmRadial                                & 59. ctree2                                          & 60. evtree                               \\
                        61. pda                                  & 62. rpart                                    & 63. cforest                                         & 64. svmLinear                            \\
                        65. xyf                                  & 66. C5.0Tree                                 & 67. avNNet                                          & 68. kknn                                 \\
                        69. svmRadialCost                        & 70. gaussprRadial                            & 71. FH.GBML                                         & 72. svmLinear2                           \\
                        73. bstSm                                & 74. LogitBoost                               & 75. wsrf                                            & 76. plr                                  \\
                        77. xgbLinear                            & 78. rf                                       & 79. null                                            & 80. protoclass                           \\
                        81. monmlp                               & 82. Rborist                                  & 83. mlpWeightDecay                                  & 84. svmRadialWeights                     \\
                        85. mlpML                                & 86. ctree                                    & 87. loclda                                          & 88. sdwd                                 \\
                        89. mlpWeightDecayML                     & 90. svmRadialSigma                           & 91. bstTree                                         & 92. dnn                                  \\
                        93. ordinalRF                            & 94. pda2                                     & 95. BstLm                                           & 96. RRFglobal                            \\
                        97. mlp                                  & 98. rpart1SE                                 & 99. pcaNNet                                         & 100. ORFsvm                              \\
                        101. parRF                               & 102. rpart2                                  & 103. gaussprPoly                                    & 104. C5.0Rules                           \\
                        105. rda                                 & 106. rbfDDA                                  & 107. multinom                                       & 108. gaussprLinear                       \\
                        109. svmPoly                             & 110. knn                                     & 111. treebag                                        & 112. RRF                                 \\
                        113. ctree2                              & 114. evtree                                  & 115. pda                                            & 116. rpart                               \\
                        117. cforest                             & 118. xyf                                     & 119. C5.0Tree                                       & 120. kknn                                \\
                        121. gaussprRadial                       & 122. LogitBoost                              & 123. wsrf                                           & 124. xgbLinear                           \\
                        125. rf                                  & 126. null                                    & 127. monmlp                                         & 128. Rborist                             \\
                        129. mlpWeightDecay                      & 130. mlpML                                   & 131. ctree                                          & 132. mlpWeightDecayML                    \\
                        133. dnn                                 & 134. pda2                                    & 135. RRFglobal                                      & 136. mlp                                 \\
                        137. rpart1SE                            & 138. parRF                                   & 139. rpart2                                         & 140. gaussprPoly                         \\
                        141. C5.0Rules                           & 142. rbfDDA                                  & 143. multinom                                       & 144. gaussprLinear                       \\
                        145. binaryConsistency                   & 146. chiSquared                              & 147. cramer                                         & 148. gainRatio                           \\
                        149. giniIndex                           & 150. IEConsistency                           & 151. IEPConsistency                                 & 152. mutualInformation                   \\
                        153. roughsetConsistency                 & 154. ReliefFeatureSetMeasure                 & 155. symmetricalUncertain                           & 156. \itemcolortwo{IteratedEstimator}    \\
                        157. \itemcolortwo{PermutationEstimator} & 158. \itemcolortwo{KernelEstimator}          & 159. \itemcolortwo{SignEstimator}                   & 160. \itemcolorone{Shapley}              \\
                        161. \itemcolorone{Banzhaf}              & 162. RF                                      & 163. \itemcolornine{Garson}                         & 164. \itemcolornine{VIANN}               \\
                        165. \itemcolornine{LOFO}                & 166. Relief                                  & 167. ReliefF                                        & 168. RReliefF                            \\
                        169. fit criterion measure               & 170. f ratio measure                         & 171. gini index                                     & 172. su measure                          \\
                        173. spearman corr                       & 174. pearson corr                            & 175. fechner corr                                   & 176. kendall corr                        \\
                        177. chi2 measure                        & 178. anova                                   & 179. laplacian score                                & 180. information gain                    \\
                        181. modified t score                    & 182. MIM                                     & 183. MRMR                                           & 184. JMI                                 \\               185. Add: CIFE                           &
                        186. CMIM                                & 187. ICAP                                    & 188. DCSF                                                                                      \\          189. CFR                                 &
                        190. MRI                                 & 191. IWFS                                    & 192. NDFS                                                                                      \\         193. RFS                              &
                        194. SPEC                                & 195. MCFS                                    & 196. UDFS                                                                                      \\  197. R2 & 198. DC & 199. BCDC & 200. AIDC \\ 201. HSIC &     202. \textbf{\bpfi{}}
                    \end{tabular}
                };


                \draw[draw = boxcolor] (0,0) rectangle (\collfour, 53 *\rowh + 0.06);

                \tikzset{shift = {(0,-10*\rowh)}}
                \node[align = center, anchor = south] at (\collfour / 2,0){
                    \resizebox*{1.95\linewidth}{!}{
                        \footnotesize
                        \begin{tabular}{RLRRLRRLR}\topprule
                            \multicolumn{9}{C}{\headcellcolor \headfont \bf \normalsize Legend}                                                                                                                                                                                                                                                                                \\ \middrule
                            \trowcolor \trowfont \cellcolor{colorback1} 1-12     & sklearn                     & \cite{scikit-learn} & \cellcolor{colorback2} 13-20    & Additional methods          & \cite{Abe2005, Rinfotheo, Fisher, Ghorbani2022, RPartykit, RVita, Treeinterpreter, Diffi} & \cellcolor{colorback3} 21-27    & shap explainer             & \cite{Shap}  \\
                            \trowcolor \trowfont \cellcolor{colorback4} 28-29    & Relative feature importance & \cite{Onig2021}     & \cellcolor{colorback5} 30-32    & \textsf{R} vip              & \cite{Rvip}                                                                               & \cellcolor{colorback6} 33-44    & scipy stats                & \cite{Scipy} \\
                            \trowcolor \trowfont \cellcolor{colorback7} 45-47    & booster classifier          & \cite{XGBoost}      & \cellcolor{colorback8} 48-109   & \textsf{R} caret classifier & \cite{Rcaret}                                                                             & \cellcolor{colorback9} 110-144  & \textsf{R} firm classifier & \cite{Rvip}  \\
                            \trowcolor \trowfont \cellcolor{colorback10} 145-155 & \textsf{R} FSinR Classifier & \cite{RFSinR}       & \cellcolor{colorback11} 156-159 & Sage Classifier             & \cite{Sage}                                                                               & \cellcolor{colorback12} 160-161 & QII Averaged Classifier    & \cite{qii}   \\
                            \trowcolor \trowfont \cellcolor{colorback13} 162-165 & Rebelosa Classifier         & \cite{Rebelosa}     & \cellcolor{colorback14} 166-168 & Relief Classifier           & \cite{Relief}                                                                             & \cellcolor{colorback15} 169-196 & ITMO                       & \cite{ITMO}  \\
                            \trowcolor \trowfont \cellcolor{colorback16} 197-201 & Sunnies                     & \cite{Fryer2021}    & \cellcolor{colorback17} 202     & \textbf{\bpfi{}}            & -                                                                                         &                                 &                            &              \\ \bottommrule
                        \end{tabular}
                    }
                };
            \end{tikzpicture}
        }
    \end{center}
\end{table*}
\endgroup

\subsection{Synthetic datasets} \label{feature importance: subsec: Synthetic datasets}


Next, we briefly discuss the datasets that are used to test the properties described in \Cref{feature importance: sec: Properties of BP-FI} for alternative FI methods. In \Cref{feature importance: appendix: datasets}, we introduce each dataset and explain how they are generated. To draw fair conclusions, the datasets are not drawn randomly, but \emph{fixed}. To give an example of how we do generate a dataset, we examine \Cref{feature importance: dataset: 1} \emph{Binary system} (see \Cref{feature importance: appendix: datasets}), where the target variable $Y$ is defined as $Y := \sum_{i=1}^{3} 2^{i-1} \cdot X_i$ with $X_i\sim\UU{\{0,1\}}$ for all $i \in \{1,2,3\}$. To get interpretable results, we draw each combination of $X$ and $Y$ values the \emph{same number} of times. An example can be seen in \Cref{feature importance: tab: difference drawn uniform}. For most datasets, we draw 1{,}000 samples in total. However \Cref{feature importance: dataset: 6,feature importance: dataset: 7} consist of 2{,}000 samples to ensure null-independence. The datasets have been selected to be computationally inexpensive and to test many properties (see \Cref{feature importance: subsec: Property evaluation}) with a limited number of datasets. An overview of the generated datasets can be found in \Cref{feature importance: tab: summary datasets} including the corresponding outcome of \bpfi{}. \Cref{feature importance: appendix: datasets} provides more technical details about the features and target variables.

\begin{table}
    \caption{\textbf{Fixed draw:} Example of how the datasets are drawn. Instead of drawing each possible outcome uniformly at random, we draw each combination an equal fixed number of times.}
    \label{feature importance: tab: difference drawn uniform}
    \begin{center}
        \begin{adjustbox}{width=0.4\linewidth,keepaspectratio}
            \begin{tabular}{*{6}{C}} \topprule
                \multicolumn{4}{c}{\headcellcolor \headfont \textbf{Outcome}} & \multicolumn{2}{c}{\headcellcolor \headfont \textbf{\# Drawn}}                                                                                     \\
                \headrowcolor \subheadfont $X_1$                              & \subheadfont $X_2$                                             & \subheadfont $X_3$ & \subheadfont $Y$ & \subheadfont Fixed & \subheadfont Uniform \\ \middrule
                \trowcolor \trowfont 0                                        & 0                                                              & 0                  & 0                & 125                & 133                  \\
                \trowcolor \trowfont 0                                        & 0                                                              & 1                  & 4                & 125                & 129                  \\
                \trowcolor \trowfont 0                                        & 1                                                              & 0                  & 2                & 125                & 121                  \\
                \trowcolor \trowfont 0                                        & 1                                                              & 1                  & 6                & 125                & 109                  \\
                \trowcolor \trowfont 1                                        & 0                                                              & 0                  & 1                & 125                & 136                  \\
                \trowcolor \trowfont 1                                        & 0                                                              & 1                  & 5                & 125                & 124                  \\
                \trowcolor \trowfont 1                                        & 1                                                              & 0                  & 3                & 125                & 115                  \\
                \trowcolor \trowfont 1                                        & 1                                                              & 1                  & 7                & 125                & 133                  \\ \bottommrule
            \end{tabular}
        \end{adjustbox}
    \end{center}
\end{table}

\begin{sidewaystable*}
    \caption{\textbf{Overview of datasets:} An overview of the generated datasets and the corresponding \bpfi{} outcome. The details of these datasets can be found in \Cref{feature importance: appendix: datasets}. They are used to evaluate if existing FI methods adhere to the same properties as \bpfi{} (see \Cref{feature importance: subsec: Property evaluation}).}
    \label{feature importance: tab: summary datasets}
    \begin{center}
        \begin{adjustbox}{width=1.0\linewidth,keepaspectratio}
            \def\arraystretch{1.5}
            \begin{tabular}{LRLLL} \topprule
                \multicolumn{3}{c}{\headcellcolor \headfont \textbf{Dataset}} & \headcellcolor \headfont \textbf{Variables} & \headcellcolor \headfont \textbf{\bpfi{} outcome}                                                                                                                                                                                                     \\ \middrule
                \trowcolor \trowfont \textbf{Binary system}                   & 1.                                          & - base                                            & $(X_1, X_2,X_3)$                                                                                                                    & $\left (0.333, 0.333, 0.333 \right )$                       \\
                \trowcolor \trowfont                                          & 2.                                          & - clone                                           & $(X_1^{\text{clone}}, X_1, X_2,X_3)$                                                                                                & $\left (0.202, 0.202, 0.298, 0.298 \right )$                \\
                \trowcolor \trowfont                                          & 3.                                          & - clone + 1x fully info.                          & $(X_1^{\text{clone}}, X_1, X_2,X_3, X_4^{\text{full}})$                                                                             & $\left ( 0.148, 0.148, 0.183, 0.183, 0.338 \right )$        \\
                \trowcolor \trowfont                                          & 4.                                          & - clone + 2x fully info.                          & $(X_1^{\text{clone}}, X_1, X_2,X_3, X_4^{\text{full}}, X_5^{\text{full}})$                                                          & $\left ( 0.117, 0.117, 0.136, 0.136, 0.248, 0.248 \right )$ \\
                \trowcolor \trowfont                                          & 5.                                          & - clone + 2x fully info. (different order)        & $(X_3, X_4^{\text{full}}, X_5^{\text{full}}, X_1^{\text{clone}}, X_1, X_2)$                                                         & $\left ( 0.136, 0.248, 0.248, 0.117, 0.117, 0.136 \right )$ \\ \middddrule
                \trowcolor \trowfont \textbf{Null-independent system}         & 6.                                          & - base                                            & $(X_1^{\text{null-indep.}},X_2^{\text{null-indep.}},X_3^{\text{null-indep.}})$                                                      & $\left ( 0.000, 0.000, 0.000 \right )$                      \\
                \trowcolor \trowfont                                          & 7.                                          & - constant variable                               & $(X_1^{\text{null-indep.}},X_2^{\text{null-indep.}},X_3^{\text{null-indep.}}, X_4^{\text{const, null-indep.}})$                     & $\left ( 0.000, 0.000, 0.000, 0.000 \right )$               \\ \middddrule
                \trowcolor \trowfont \textbf{Increasing bins}                 & 8.                                          & - base                                            & $(X_1^{\text{bins}=10}, X_2^{\text{bins}=50}, X_3^{\text{bins}=1{,}000,~\text{full}})$                                              & $\left ( 0.297, 0.342, 0.361 \right )$                      \\
                \trowcolor \trowfont                                          & 9.                                          & - more variables                                  & $(X_1^{\text{bins}=10}, X_2^{\text{bins}=20}, X_3^{\text{bins}=50}, X_4^{\text{bins}=100}, X_5^{\text{bins}=1{,}000,~\text{full}})$ & $\left ( 0.179, 0.193, 0.204, 0.208, 0.216 \right )$        \\
                \trowcolor \trowfont                                          & 10.                                         & - clone (different order)                         & $(X_3^{\text{bins}=1{,}000,~\text{full}}, X_2^{\text{bins}=50}, X_1^{\text{bins}=10}, X_3^{\text{clone},~\text{full}})$             & $\left ( 0.262, 0.253, 0.223, 0.262 \right )$               \\ \middddrule
                \trowcolor \trowfont \textbf{Dependent system}                & 11.                                         & - 1x fully info.                                  & $(X_1^{\text{full}}, X_2^{\text{null-indep.}}, X_3^{\text{null-indep.}})$                                                           & $\left ( 1.000, 0.000, 0.000 \right )$                      \\
                \trowcolor \trowfont                                          & 12.                                         & - 2x fully info.                                  & $(X_1^{\text{full}}, X_2^{\text{full}}, X_3^{\text{null-indep.}})$                                                                  & $\left ( 0.500, 0.500, 0.000 \right )$                      \\
                \trowcolor \trowfont                                          & 13.                                         & - 3x fully info.                                  & $(X_1^{\text{full}}, X_2^{\text{full}}, X_3^{\text{full}})$                                                                         & $\left ( 0.333, 0.333, 0.333 \right )$                      \\ \middddrule
                \trowcolor \trowfont \textbf{XOR dataset}                     & 14.                                         & - base                                            & $(X_1, X_2)$                                                                                                                        & $\left ( 0.500, 0.500 \right )$                             \\
                \trowcolor \trowfont                                          & 15.                                         & - single variable                                 & $(X_1^{\text{null-indep.}})$                                                                                                        & $\left ( 0.000 \right )$                                    \\
                \trowcolor \trowfont                                          & 16.                                         & - clone                                           & $(X_1^{\text{clone}}, X_1, X_2)$                                                                                                    & $\left (0.167, 0.167, 0.667 \right )$                       \\
                \trowcolor \trowfont                                          & 17.                                         & - null-independent                                & $(X_1, X_2, X_3^{\text{null-indep.}})$                                                                                              & $\left ( 0.500, 0.500, 0.000 \right )$                      \\ \middddrule
                \trowcolor \trowfont \textbf{Probability dataset}             & 18-28.                                      & - for $p \in \{0, 0.1, \dots, 1\} $               & $(X_1, X_2) $                                                                                                                       & $\left ( p, 1-p \right )$
                \\ \bottommrule
            \end{tabular}
        \end{adjustbox}
    \end{center}
\end{sidewaystable*}

\subsection{Property evaluation} \label{feature importance: subsec: Property evaluation}


In \Cref{feature importance: subsec: Alternative FI methods}, we gathered a collection of existing FI methods. In this section, we evaluate if these FI methods have the same desirable and proven properties of the \bpfi{} method (see \Cref{feature importance: sec: Properties of BP-FI}). Due to the sheer number of FI methods (468), it is unfeasible to prove each property for every method. Instead, we devise tests to find counterexamples of these properties using generated datasets (see \Cref{feature importance: subsec: Synthetic datasets}). Due to the number of tests (18), we only discuss the parts that are not straightforward, as most test directly measure the corresponding property. An overview of each test can be found in \Cref{feature importance: appendix: tests}. A summary of the tests can be found in \Cref{feature importance: tab: summary experiments}, where it is outlined for each test which property is tested on which datasets.

\begin{sidewaystable*}
        \caption{\textbf{Overview of experiments:} To evaluate if existing FI methods have the same properties as the \bpfi{}, we use the tests from \Cref{feature importance: appendix: tests} on the datasets from \Cref{feature importance: appendix: datasets}. \cmark means that the test is performed on this dataset. \basemark (i) denotes that this dataset is used as baseline or in conjunction with dataset $i$. The details of the tests and datasets can be found in the appendix.}
        \label{feature importance: tab: summary experiments}
        \begin{adjustbox}{width=\linewidth,keepaspectratio}
            \def\x{28}
            \begin{tabular}{F{3 cm}C*{\x}{C}} \topprule
                \headrowcolor \tablerowfontcolor \textbf{Test}                                  & \textbf{Evaluates:}                                             & \multicolumn{28}{C}{\headfont \headcellcolor \textbf{Dataset (\Cref*{feature importance: appendix: datasets})} }                                                                                                                                                                                                                                                                                                                                                                                                                                                                                                                     \\
                \headrowcolor \tablerowfontcolor  (\Cref*{feature importance: appendix: tests}) & Property/Corollary                                              & \subheadfont 1                                                                                                   & \subheadfont 2       & \subheadfont 3      & \subheadfont 4 & \subheadfont 5 & \subheadfont 6 & \subheadfont 7 & \subheadfont 8  & \subheadfont 9 & \subheadfont 10 & \subheadfont 11   & \subheadfont 12 & \subheadfont 13 & \subheadfont 14         & \subheadfont 15 & \subheadfont 16 & \subheadfont 17 & \subheadfont 18 & \subheadfont 19 & \subheadfont 20 & \subheadfont 21 & \subheadfont 22 & \subheadfont 23 & \subheadfont 24 & \subheadfont 25 & \subheadfont 26 & \subheadfont 27 & \subheadfont 28 \\ \middrule
                \trowcolor \trowfont \ref{feature importance: test: 1}                          & \ref{feature importance: prop: efficiency}                      & \cmark                                                                                                           & \cmark               & \cmark              & \cmark         & \cmark         & \cmark         & \cmark         & \cmark          & \cmark         & \cmark          & \cmark            & \cmark          & \cmark          & \cmark                  & \cmark          & \cmark          & \cmark          & \cmark          & \cmark          & \cmark          & \cmark          & \cmark          & \cmark          & \cmark          & \cmark          & \cmark          & \cmark          & \cmark          \\
                \trowcolor \trowfont \ref{feature importance: test: 2}                          & \ref{feature importance: corollary: sum stable}                 & \basemark (2-5)                                                                                                  & \cmark               & \cmark              & \cmark         & \cmark         & \basemark (7)  & \cmark         & \basemark(9-10) & \cmark         & \cmark          & \basemark (12-13) & \cmark          & \cmark          & \basemark(16-17)        &                 & \cmark          & \cmark          &                 &                 &                 &                 &                 &                 &                 &                 &                 &                 &                 \\
                \trowcolor \trowfont \ref{feature importance: test: 3}                          & \ref{feature importance: prop: symmetry}                        & \cmark                                                                                                           & \cmark               & \cmark              & \cmark         & \cmark         & \cmark         & \cmark         &                 &                & \cmark          & \cmark            & \cmark          & \cmark          & \cmark                  &                 & \cmark          & \cmark          &                 &                 &                 &                 &                 & \cmark          &                 &                 &                 &                 &                 \\
                \trowcolor \trowfont \ref{feature importance: test: 4}                          & \ref{feature importance: prop: range}                           & \cmark                                                                                                           & \cmark               & \cmark              & \cmark         & \cmark         & \cmark         & \cmark         & \cmark          & \cmark         & \cmark          & \cmark            & \cmark          & \cmark          & \cmark                  & \cmark          & \cmark          & \cmark          & \cmark          & \cmark          & \cmark          & \cmark          & \cmark          & \cmark          & \cmark          & \cmark          & \cmark          & \cmark          & \cmark          \\
                \trowcolor \trowfont \ref{feature importance: test: 5}                          & \ref{feature importance: prop: range}                           & \cmark                                                                                                           & \cmark               & \cmark              & \cmark         & \cmark         & \cmark         & \cmark         & \cmark          & \cmark         & \cmark          & \cmark            & \cmark          & \cmark          & \cmark                  & \cmark          & \cmark          & \cmark          & \cmark          & \cmark          & \cmark          & \cmark          & \cmark          & \cmark          & \cmark          & \cmark          & \cmark          & \cmark          & \cmark          \\
                \trowcolor \trowfont \ref{feature importance: test: 6}                          & \ref{feature importance: prop: bounds}                          & \cmark                                                                                                           & \cmark               & \cmark              & \cmark         & \cmark         & \cmark         & \cmark         & \cmark          & \cmark         & \cmark          & \cmark            & \cmark          & \cmark          & \cmark                  & \cmark          & \cmark          & \cmark          & \cmark          & \cmark          & \cmark          & \cmark          & \cmark          & \cmark          & \cmark          & \cmark          & \cmark          & \cmark          & \cmark          \\
                \trowcolor \trowfont \ref{feature importance: test: 7}                          & \ref{feature importance: prop: bounds}                          & \cmark                                                                                                           & \cmark               & \cmark              & \cmark         & \cmark         & \cmark         & \cmark         & \cmark          & \cmark         & \cmark          & \cmark            & \cmark          & \cmark          & \cmark                  & \cmark          & \cmark          & \cmark          & \cmark          & \cmark          & \cmark          & \cmark          & \cmark          & \cmark          & \cmark          & \cmark          & \cmark          & \cmark          & \cmark          \\
                \trowcolor \trowfont \ref{feature importance: test: 8}                          & \ref{feature importance: prop: zero FI}                         &                                                                                                                  &                      &                     &                &                & \cmark         & \cmark         &                 &                &                 & \cmark            & \cmark          &                 &                         & \cmark          &                 & \cmark          & \cmark          &                 &                 &                 &                 &                 &                 &                 &                 &                 & \cmark          \\
                \trowcolor \trowfont \ref{feature importance: test: 9}                          & \ref{feature importance: prop: zero FI}                         & \cmark                                                                                                           & \cmark               & \cmark              & \cmark         & \cmark         & \cmark         & \cmark         & \cmark          & \cmark         & \cmark          & \cmark            & \cmark          & \cmark          & \cmark                  & \cmark          & \cmark          & \cmark          & \cmark          & \cmark          & \cmark          & \cmark          & \cmark          & \cmark          & \cmark          & \cmark          & \cmark          & \cmark          & \cmark          \\
                \trowcolor \trowfont \ref{feature importance: test: 10}                         & \ref{feature importance: prop: FI equal to one}                 &                                                                                                                  &                      &                     &                &                &                &                &                 &                &                 & \cmark            &                 &                 &                         &                 &                 &                 &                 &                 &                 &                 &                 &                 &                 &                 &                 &                 &                 \\
                \trowcolor \trowfont \ref{feature importance: test: 11}                         & \ref{feature importance: prop: max Fi when fully determined}    &                                                                                                                  &                      & \cmark              & \cmark         & \cmark         &                &                & \cmark          & \cmark         & \cmark          & \cmark            & \cmark          & \cmark          &                         &                 &                 &                 & \cmark          &                 &                 &                 &                 &                 &                 &                 &                 &                 & \cmark          \\
                \trowcolor \trowfont \ref{feature importance: test: 12}                         & \ref{feature importance: prop: limiting outcome space}          &                                                                                                                  &                      &                     &                &                &                &                & \cmark          & \cmark         & \cmark          &                   &                 &                 &                         &                 &                 &                 &                 &                 &                 &                 &                 &                 &                 &                 &                 &                 &                 \\
                \trowcolor \trowfont \ref{feature importance: test: 13}                         & \ref{feature importance: prop: adding features can increase FI} & \basemark (2)                                                                                                    & \cmark \basemark (3) & \cmark \basemark(4) & \cmark         &                & \basemark (7)  & \cmark         & \basemark(9-10) & \cmark         & \cmark          &                   &                 &                 & \cmark \basemark(16-17) & \basemark (14)  & \cmark          & \cmark          &                 &                 &                 &                 &                 &                 &                 &                 &                 &                 &                 \\
                \trowcolor \trowfont \ref{feature importance: test: 14}                         & \ref{feature importance: prop: adding features can decrease FI} & \basemark (2)                                                                                                    & \cmark \basemark (3) & \cmark \basemark(4) & \cmark         &                & \basemark (7)  & \cmark         & \basemark(9-10) & \cmark         & \cmark          &                   &                 &                 & \cmark \basemark(16-17) & \basemark (14)  & \cmark          & \cmark          &                 &                 &                 &                 &                 &                 &                 &                 &                 &                 &                 \\
                \trowcolor \trowfont \ref{feature importance: test: 15}                         & \ref{feature importance: prop: cloning does not increase FI}    & \basemark (2)                                                                                                    & \cmark               &                     &                &                &                &                & \basemark (10)  &                & \cmark          &                   &                 &                 &                         &                 & \basemark (16)  &                 & \cmark          &                 &                 &                 &                 &                 &                 &                 &                 &                 &                 \\
                \trowcolor \trowfont \ref{feature importance: test: 16}                         & \ref{feature importance: prop: order does not change FI}        &                                                                                                                  &                      &                     & \basemark (5)  & \cmark         &                &                &                 &                &                 &                   &                 &                 &                         &                 &                 &                 & \basemark (28)  & \basemark (27)  & \basemark (26)  & \basemark (25)  & \basemark (24)  &                 & \cmark          & \cmark          & \cmark          & \cmark          & \cmark          \\
                \trowcolor \trowfont \ref{feature importance: test: 17}                         & \ref{feature importance: prop: xor dataset}                     &                                                                                                                  &                      &                     &                &                &                &                &                 &                &                 &                   &                 &                 & \cmark                  &                 &                 & \cmark          &                 &                 &                 &                 &                 &                 &                 &                 &                 &                 &                 \\
                \trowcolor \trowfont \ref{feature importance: test: 18}                         & \ref{feature importance: prop: probability dataset}             &                                                                                                                  &                      &                     &                &                &                &                &                 &                &                 &                   &                 &                 &                         &                 &                 &                 & \cmark          & \cmark          & \cmark          & \cmark          & \cmark          & \cmark          & \cmark          & \cmark          & \cmark          & \cmark          & \cmark          \\
                \bottommrule
            \end{tabular}
        \end{adjustbox}
\end{sidewaystable*}
\paragraph*{Computational errors}
To allow for computational errors, we tolerate a margin of $\epsilon = 0.01$ in each test. If, e.g., an FI value should be zero, a score of $0.01$ or $-0.01$ is still considered a \emph{pass}, whereas an FI value of 0.05 is counted as a \emph{fail}. Usually, this works in the favor of the FI method. However, in \Cref{feature importance: test: 9} we evaluate if the FI method assigns zero FI to variables that are not null-independent. In this case, we consider $\vert \FI{X} \vert \leq \epsilon$ to be \emph{zero}, as the datasets are constructed in such a way that variables are either null-independent or far from being null-independent.

\paragraph*{Running time}
We limit the running time to one hour per dataset on an i7-12700K processor, whilst four algorithms are running simultaneously. The datasets consist of a small number of features with a very limited outcome space and the number of samples is either 1{,}000 or 2{,}000, which is why one hour is a reasonable amount of time.

\paragraph*{NaN or infinite values}
In some cases, an FI method assigns NaN or $\pm \infty$ to a feature. How we handle these values depends on the test. E.g., we consider NaN to fall outside the range $[0,1]$ (\Cref{feature importance: test: 4,feature importance: test: 5}{5}), but when we evaluate if the sum of FI values remains stable (\Cref{feature importance: test: 2}) or if two symmetric features receive the same FI (\Cref{feature importance: test: 3}), we consider twice NaN or twice $\pm \infty$ to be the same.

\paragraph*{Property 9 (Limiting the outcome space)}
\Cref{feature importance: prop: limiting outcome space} states that applying any measurable function $f$ to a RV $X$ cannot increase the FI. In other words, $\FI{X} \geq \FI{f(X)}$ holds. This property is tested using \Cref{feature importance: dataset: 8,feature importance: dataset: 9,feature importance: dataset: 10} (see \Cref{feature importance: tab: summary experiments}). These datasets contain variables that are the outcome of binning the target variable using different number of bins. This is how \Cref{feature importance: prop: limiting outcome space} is tested, as it should hold that $\FI{X_i} \geq \FI{X_j}$, whenever $X_i$ has more bins than $X_j$.

\paragraph*{Properties 11 and 12 (Adding features can increase/decrease FI)}
In all other tests, the goal is to find a counterexample of the property. However, \Cref{feature importance: test: 13,feature importance: test: 14} are designed to evaluate if a feature gets an increased/decreased FI when a feature is added. This increase/decrease should be more than $\epsilon$. The datasets are chosen in such a way that both an increase and decrease could occur (according to the \bpfi{}). Only for these tests, we consider the test failed if no counterexample (increase/decrease) is found.

\subsection{Evaluation results} \label{feature importance: subsec: Evaluation results}


An overview of the general results can be seen in \Cref{feature importance: tab: summary passed}, where the number of methods that \emph{pass} and \emph{fail} is given per test. Next, we highlight additional insights into the results of the experiments.

\paragraph*{Best performing methods}
The top 20 FI methods that pass the most tests are given in \Cref{feature importance: tab: best 20}. Out of 18 tests, the \bpfi{} passes all tests, which is as expected as we have proven in \Cref{feature importance: sec: Properties of BP-FI} that the \bpfi{} actually has these properties. Classifiers from \emph{\textsf{R} FSinR Classifier} and \emph{ITMO} fill 11 of the top 20 spots. Out of 11 \textsf{R} FSinR Classifier methods, six are in the top 20, which is quite remarkable. However, observe that the gap between the \bpfi{} method and the second best method is $18-11 = 7$ passed tests. Additionally, 424 out of 468 methods fail more than half of the tests.
\Cref{feature importance: fig: histogram passed test} shows how frequently each number of passed tests occurs. A detailed overview of where each top 20 method fails, can be seen in \Cref{feature importance: tab: summary passed}. Note again that in \Cref{feature importance: test: 13,feature importance: test: 14} it is considered a fail if adding features never increase or decrease the FI, respectively. It could be that these methods are in fact capable of increasing or decreasing, but for some reason do not with our datasets. Strikingly, most of these methods perform bad on the datasets with a desirable outcome (\Cref{feature importance: test: 17,feature importance: test: 18}). Adding a variable without additional information (\Cref{feature importance: test: 2}), also often leads to a change in total FI.

\begin{center}
    \begin{table*}
        \caption{\textbf{Overview of the results:} Each FI method is evaluated using the tests outlined in \Cref{feature importance: appendix: tests}, which evaluates if the method adheres to the same properties as the \bpfi{} (see \Cref{feature importance: sec: Properties of BP-FI}). This table summarizes out of 468 FI methods how many \emph{pass} or \emph{fail} the test. A distinction is made for the top 20 passing methods. Failing the test means that a counterexample is found. Note that passing the test does not `prove' that the FI method actually has the property. \emph{No result} indicates that the test could not be executed, because the running time of the FI method was too long or an error occurred.}
        \label{feature importance: tab: summary passed}
        \begin{adjustbox}{width=\linewidth,keepaspectratio}
            \def\x{18}
            \begin{tabular}{F{3 cm}*{\x}{R}} \topprule
                \headrowcolor \tablerowfontcolor                     & \multicolumn{\x}{C}{\textbf{Test}}                                                                                                                                                                                                                                                                                                                                                                                                                                                                                                                                                                                           \\
                \headrowcolor \tablerowfontcolor                     & \subheadfont 1                     & \subheadfont 2                 & \subheadfont 3                 & \subheadfont 4                 & \subheadfont 5                 & \subheadfont 6                 & \subheadfont 7                 & \subheadfont 8                 & \subheadfont 9                 & \subheadfont 10                 & \subheadfont 11                 & \subheadfont 12                 & \subheadfont 13                 & \subheadfont 14                 & \subheadfont 15                 & \subheadfont 16                 & \subheadfont 17                 & \subheadfont 18                 \\ \middrule
                \trowcolor \trowfont {\subheadfont \textbf{Overall}} &                                    &                                &                                &                                &                                &                                &                                &                                &                                &                                 &                                 &                                 &                                 &                                 &                                 &                                 &                                                                   \\
                \trowcolor \trowfont \# Passed                       & \var{Test_1_passed_test}           & \var{Test_2_passed_test}       & \var{Test_3_passed_test}       & \var{Test_4_passed_test}       & \var{Test_5_passed_test}       & \var{Test_6_passed_test}       & \var{Test_7_passed_test}       & \var{Test_8_passed_test}       & \var{Test_9_passed_test}       & \var{Test_10_passed_test}       & \var{Test_11_passed_test}       & \var{Test_12_passed_test}       & \var{Test_13_passed_test}       & \var{Test_14_passed_test}       & \var{Test_15_passed_test}       & \var{Test_16_passed_test}       & \var{Test_17_passed_test}       & \var{Test_18_passed_test}
                \\
                \trowcolor \trowfont \# Failed                       & \var{Test_1_failed_test}           & \var{Test_2_failed_test}       & \var{Test_3_failed_test}       & \var{Test_4_failed_test}       & \var{Test_5_failed_test}       & \var{Test_6_failed_test}       & \var{Test_7_failed_test}       & \var{Test_8_failed_test}       & \var{Test_9_failed_test}       & \var{Test_10_failed_test}       & \var{Test_11_failed_test}       & \var{Test_12_failed_test}       & \var{Test_13_failed_test}       & \var{Test_14_failed_test}       & \var{Test_15_failed_test}       & \var{Test_16_failed_test}       & \var{Test_17_failed_test}       & \var{Test_18_failed_test}
                \\
                \trowcolor \trowfont \# No result                    & \var{Test_1_no_result}             & \var{Test_2_no_result}         & \var{Test_3_no_result}         & \var{Test_4_no_result}         & \var{Test_5_no_result}         & \var{Test_6_no_result}         & \var{Test_7_no_result}         & \var{Test_8_no_result}         & \var{Test_9_no_result}         & \var{Test_10_no_result}         & \var{Test_11_no_result}         & \var{Test_12_no_result}         & \var{Test_13_no_result}         & \var{Test_14_no_result}         & \var{Test_15_no_result}         & \var{Test_16_no_result}         & \var{Test_17_no_result}         & \var{Test_18_no_result}
                \\ \middddrule
                \trowcolor \trowfont {\subheadfont \textbf{Top 20}}  &                                    &                                &                                &                                &                                &                                &                                &                                &                                &                                 &                                 &                                 &                                 &                                 &                                 &                                 &                                                                   \\

                \trowcolor \trowfont \# Passed                       & \var{table_1_test_1_passed}        & \var{table_1_test_2_passed}    & \var{table_1_test_3_passed}    & \var{table_1_test_4_passed}    & \var{table_1_test_5_passed}    & \var{table_1_test_6_passed}    & \var{table_1_test_7_passed}    & \var{table_1_test_8_passed}    & \var{table_1_test_9_passed}    & \var{table_1_test_10_passed}    & \var{table_1_test_11_passed}    & \var{table_1_test_12_passed}    & \var{table_1_test_13_passed}    & \var{table_1_test_14_passed}    & \var{table_1_test_15_passed}    & \var{table_1_test_16_passed}    & \var{table_1_test_17_passed}    & \var{table_1_test_18_passed}
                \\
                \trowcolor \trowfont \# Failed                       & \var{table_1_test_1_failed}        & \var{table_1_test_2_failed}    & \var{table_1_test_3_failed}    & \var{table_1_test_4_failed}    & \var{table_1_test_5_failed}    & \var{table_1_test_6_failed}    & \var{table_1_test_7_failed}    & \var{table_1_test_8_failed}    & \var{table_1_test_9_failed}    & \var{table_1_test_10_failed}    & \var{table_1_test_11_failed}    & \var{table_1_test_12_failed}    & \var{table_1_test_13_failed}    & \var{table_1_test_14_failed}    & \var{table_1_test_15_failed}    & \var{table_1_test_16_failed}    & \var{table_1_test_17_failed}    & \var{table_1_test_18_failed}
                \\
                \trowcolor \trowfont \# No result                    & \var{table_1_test_1_no_result}     & \var{table_1_test_2_no_result} & \var{table_1_test_3_no_result} & \var{table_1_test_4_no_result} & \var{table_1_test_5_no_result} & \var{table_1_test_6_no_result} & \var{table_1_test_7_no_result} & \var{table_1_test_8_no_result} & \var{table_1_test_9_no_result} & \var{table_1_test_10_no_result} & \var{table_1_test_11_no_result} & \var{table_1_test_12_no_result} & \var{table_1_test_13_no_result} & \var{table_1_test_14_no_result} & \var{table_1_test_15_no_result} & \var{table_1_test_16_no_result} & \var{table_1_test_17_no_result} & \var{table_1_test_18_no_result}
                \\
                \bottommrule
            \end{tabular}
        \end{adjustbox}
    \end{table*}
\end{center}

\begin{table}
    \begin{center}
        \caption{\textbf{Top 20:} Out of 468 FI methods, these 20 methods pass the 18 tests given in \Cref{feature importance: appendix: tests} the most often. These tests are designed to examine if an FI method adheres to the same properties as the \bpfi{} , given in \Cref{feature importance: sec: Properties of BP-FI}. \emph{Passed} means that the datasets from \Cref{feature importance: appendix: datasets} do not give a counterexample. Certainly, this does not mean that the FI method is proven to actually have this property. \emph{Failed} means that a counterexample was found. \emph{No result} indicates that the test could not be executed, because the running time of the FI method was too long or an error occurred.}
        \label{feature importance: tab: best 20}
        \begin{adjustbox}{width=0.6\linewidth,keepaspectratio}
            \def\x{3}

            \begin{tabular}{RE{4 cm}*{\x}{C}} \topprule
                \headrowcolor \tablerowfontcolor                                &                                          & \multicolumn{\x}{C}{\textbf{ Combined result:}}                                                                                              \\
                \headrowcolor \multicolumn{2}{L}{\headfont \textbf{Method}}     & \subheadfont \# Passed                   & \subheadfont  \# Failed                         & \subheadfont \# No result                                                                  \\ \middrule
                \trowcolor \trowfont \bfseries \var{table_best_method_1_index}. & \bfseries \var{table_best_method_1_name} & \bfseries \var{table_best_method_1_passed}      & \bfseries \var{table_best_method_1_failed} & \bfseries \var{table_best_method_1_no_result} \\
                \trowcolor \trowfont \var{table_best_method_2_index}.           & \var{table_best_method_2_name}           & \var{table_best_method_2_passed}                & \var{table_best_method_2_failed}           & \var{table_best_method_2_no_result}           \\
                \trowcolor \trowfont \var{table_best_method_3_index}.           & \var{table_best_method_3_name}           & \var{table_best_method_3_passed}                & \var{table_best_method_3_failed}           & \var{table_best_method_3_no_result}           \\
                \trowcolor \trowfont \var{table_best_method_4_index}.           & \var{table_best_method_4_name}           & \var{table_best_method_4_passed}                & \var{table_best_method_4_failed}           & \var{table_best_method_4_no_result}           \\
                \trowcolor \trowfont \var{table_best_method_5_index}.           & \var{table_best_method_5_name}           & \var{table_best_method_5_passed}                & \var{table_best_method_5_failed}           & \var{table_best_method_5_no_result}           \\
                \trowcolor \trowfont \var{table_best_method_6_index}.           & \var{table_best_method_6_name}           & \var{table_best_method_6_passed}                & \var{table_best_method_6_failed}           & \var{table_best_method_6_no_result}           \\
                \trowcolor \trowfont \var{table_best_method_7_index}.           & \var{table_best_method_7_name}           & \var{table_best_method_7_passed}                & \var{table_best_method_7_failed}           & \var{table_best_method_7_no_result}           \\
                \trowcolor \trowfont \var{table_best_method_8_index}.           & \var{table_best_method_8_name}           & \var{table_best_method_8_passed}                & \var{table_best_method_8_failed}           & \var{table_best_method_8_no_result}           \\
                \trowcolor \trowfont \var{table_best_method_9_index}.           & \var{table_best_method_9_name}           & \var{table_best_method_9_passed}                & \var{table_best_method_9_failed}           & \var{table_best_method_9_no_result}           \\
                \trowcolor \trowfont \var{table_best_method_10_index}.          & \var{table_best_method_10_name}          & \var{table_best_method_10_passed}               & \var{table_best_method_10_failed}          & \var{table_best_method_10_no_result}          \\
                \trowcolor \trowfont \var{table_best_method_11_index}.          & \var{table_best_method_11_name}          & \var{table_best_method_11_passed}               & \var{table_best_method_11_failed}          & \var{table_best_method_11_no_result}          \\
                \trowcolor \trowfont \var{table_best_method_12_index}.          & \var{table_best_method_12_name}          & \var{table_best_method_12_passed}               & \var{table_best_method_12_failed}          & \var{table_best_method_12_no_result}          \\
                \trowcolor \trowfont \var{table_best_method_13_index}.          & \var{table_best_method_13_name}          & \var{table_best_method_13_passed}               & \var{table_best_method_13_failed}          & \var{table_best_method_13_no_result}          \\
                \trowcolor \trowfont \var{table_best_method_14_index}.          & \var{table_best_method_14_name}          & \var{table_best_method_14_passed}               & \var{table_best_method_14_failed}          & \var{table_best_method_14_no_result}          \\
                \trowcolor \trowfont \var{table_best_method_15_index}.          & \var{table_best_method_15_name}          & \var{table_best_method_15_passed}               & \var{table_best_method_15_failed}          & \var{table_best_method_15_no_result}          \\
                \trowcolor \trowfont \var{table_best_method_16_index}.          & \var{table_best_method_16_name}          & \var{table_best_method_16_passed}               & \var{table_best_method_16_failed}          & \var{table_best_method_16_no_result}          \\
                \trowcolor \trowfont \var{table_best_method_17_index}.          & \var{table_best_method_17_name}          & \var{table_best_method_17_passed}               & \var{table_best_method_17_failed}          & \var{table_best_method_17_no_result}          \\
                \trowcolor \trowfont \var{table_best_method_18_index}.          & \var{table_best_method_18_name}          & \var{table_best_method_18_passed}               & \var{table_best_method_18_failed}          & \var{table_best_method_18_no_result}          \\
                \trowcolor \trowfont \var{table_best_method_19_index}.          & \var{table_best_method_19_name}          & \var{table_best_method_19_passed}               & \var{table_best_method_19_failed}          & \var{table_best_method_19_no_result}          \\
                \trowcolor \trowfont \var{table_best_method_20_index}.          & \var{table_best_method_20_name}          & \var{table_best_method_20_passed}               & \var{table_best_method_20_failed}          & \var{table_best_method_20_no_result}          \\
                \bottommrule
            \end{tabular}
        \end{adjustbox}
    \end{center}
\end{table}

\begin{figure*}
    \resizebox*{\linewidth}{!}{
        \begin{tikzpicture}
            \def\maxx{14};
            \def\maxy{6};
            \def\ystep{19};
            \def\xstep{1.4};
            \def\barwidth{0.357};

            \draw[fill = backgroundgray, draw = boxcolor] (0,0) rectangle (\maxx, \maxy);

            \foreach \y in {0,20,...,110} {
                    \node[anchor = east, labelcolor] at (-0.3,\y / \ystep){\y};
                    \draw[tickcolor] (-0.3, \y / \ystep) -- (0, \y / \ystep);
                }

            \foreach \y in {20,40,...,110} {
                    \draw[dashed, backgrounddash] (0, \y / \ystep) -- (\maxx, \y / \ystep);
                }

            \foreach[evaluate = {\j = int(\x - 1)}] \x in {1,2,...,19} {
                    \node[labelcolor] at (\x / \xstep , -0.6){\j};
                    \draw[labelcolor] (\x / \xstep, -0.3) -- (\x / \xstep, 0);
                }

            \node[labelcolor] at (\maxx / 2, -1.2) {\bf \small \# Passed tests};
            \node[rotate = 90, labelcolor] at (-1.5, \maxy / 2) {\bf \small Frequency};

            \node[labelcolor] (bpfi) at (17.7 / \xstep, 25 / \ystep){\small \bfseries \bpfi{}};
            \draw[->, labelcolor] (bpfi) -- (18.7 / \xstep, 2 / \ystep);

            \foreach[count = \i, evaluate = {\x = int(20-\i)}] \y in {1,0,0,0,0,0,0,5,13,25,74,78,78,101,57,24,7,4,1
                }{

                    \draw[fill = barcolor3, draw = outlinecolor2] (\x / \xstep - \barwidth, 0) -- (\x / \xstep - \barwidth, \y / \ystep) -- (\x / \xstep + \barwidth, \y / \ystep) -- (\x / \xstep + \barwidth, 0);
                    \node[above, labelcolor] at (\x / \xstep, \y / \ystep){\small \y};
                }

        \end{tikzpicture}
    }
    \caption{\textbf{Frequency of total passed test:} Histogram of the number of passed tests (out of 18) for the 468 FI methods.} \label{feature importance: fig: histogram passed test}

\end{figure*}
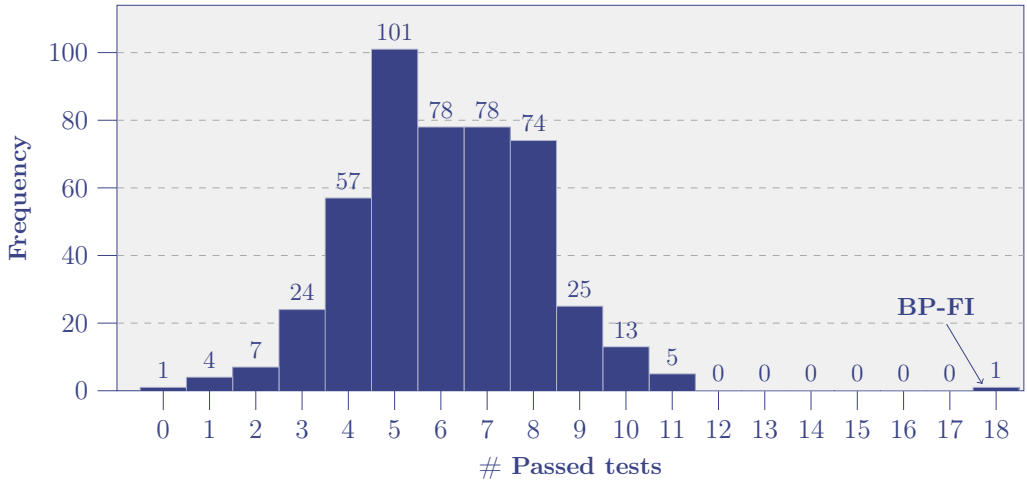

\paragraph*{Test 1}
In this test, it is evaluated if the sum of FI values is the same as the sum for \bpfi{}. At first, this seems a rather strict requirement. However, it holds for all datasets that were used that $\Dep{\FISET}{Y}$ is either zero or one. Thus, we essentially evaluate if the sum of FI is equal to one, when all variables collectively fully determine $Y$ and zero if all variables are null-independent. The tests show that no FI method is able to pass this test, except for the \bpfi{}. To highlight some of the methods that came close: {\slshape 162. Rebelosa Classifier RF, 2. Random Forest Classifier entropy, 2. Random Forest Classifier gini} only fail for the datasets where the sum should be zero (because of null-independence) and {\slshape 1. AdaBoost Classifier} only does not pass on three of the four datasets based on the XOR function (see \Cref{feature importance: appendix: datasets}), where the sum should be one, but was zero instead. FI method {\slshape 51. lssvmRadial} came closest with two fails. For the null-independent datasets (\Cref{feature importance: dataset: 6,feature importance: dataset: 7}), it gives each feature an FI of 0.5, making the sum larger than zero.

\paragraph*{Test 2}
In \Cref{feature importance: fig: unstable sum FI}, a breakdown is given of where the sum of the FI values is unstable. The most errors are made with the \emph{Binary system} datasets, when a fully informative feature is added. In total, \var{Test_2_passed_test} methods passed the test, whereas \var{Test_2_failed_test} failed. From these \var{Test_2_failed_test} methods, 279 fail with at least one increase of the sum, whereas 232 methods fail with at least one decrease. An alarming number of FI methods thus assign significantly more or less FI when a variable is added that does not contain any additional information. More or less credit is given out, whilst the collective knowledge is stable and does not warrant an increase or decrease in credit. Additionally, when the initial and final sum both contain a NaN value, it is considered as a pass. Three out of \var{Test_2_passed_test} would have not passed without this rule. If only the initial or the final sum contained NaN, it is considered a fail, because the sum is not the same. Only five methods fail solely by this rule: {\slshape 15. Fisher Score},
{\slshape 11. f classif}, {\slshape 178. anova}, {\slshape 179. laplacian score} and {\slshape 192. NDFS}.

\begin{figure*}
    \resizebox*{\linewidth}{!}{
        \begin{tikzpicture}
            \def\maxx{14};
            \def\maxy{6};
            \def\ystep{58};
            \def\xstep{0.857};
            \def\barhelp{1/\xstep}
            \def\barwidth{\barhelp / 2};

            \draw[fill = backgroundgray, draw = boxcolor] (0,0) rectangle (\maxx, \maxy);

            \foreach \y in {0,50,...,300} {
                    \node[anchor = east, labelcolor] at (-0.3,\y / \ystep){\y};
                    \draw[tickcolor] (-0.3, \y / \ystep) -- (0, \y / \ystep);
                }
            \foreach \y in {50,100,...,300} {
                    \draw[dashed, backgrounddash] (0, \y / \ystep) -- (\maxx, \y / \ystep);
                }

            \foreach \x / \l in {1 / 1\basemark 2, 2 / 1\basemark 3, 3 / 1\basemark 4, 4 / 1\basemark 5, 5 / 6\basemark 7, 6 / 8\basemark 9, 7 / 8\basemark 10, 8 / 11\basemark 12, 9 / 11\basemark 13, 10 / 14\basemark 16, 11 / 14\basemark 17} {
                    \node[labelcolor] at (\x / \xstep , -0.6){\l};
                    \draw[tickcolor] (\x / \xstep, -0.3) -- (\x / \xstep, 0);
                }

            \node[labelcolor] at (\maxx / 2, -1.2) {\bf \small Compared datasets};
            \node[rotate = 90, labelcolor] at (-1.5, \maxy / 2) {\bf \small \# Unstable sum FI};

            \foreach[count = \x] \y in {188, 302, 311, 299, 163, 190, 95, 203, 194, 124, 117
                }{

                    \draw[fill = barcolor3, draw = outlinecolor2] (\x / \xstep - \barwidth, 0) -- (\x / \xstep - \barwidth, \y / \ystep) -- (\x / \xstep + \barwidth, \y / \ystep) -- (\x / \xstep + \barwidth, 0);
                    \node[above, labelcolor] at (\x / \xstep, \y / \ystep){\small \y};
                }

        \end{tikzpicture}
    }
    \caption{\textbf{Unstable sum FI:} Whenever a variable is added that does not give any additional information, the sum of all FI should remain stable. For each comparison, we determine how often this is not the case out of 468 FI methods.} \label{feature importance: fig: unstable sum FI}

\end{figure*}
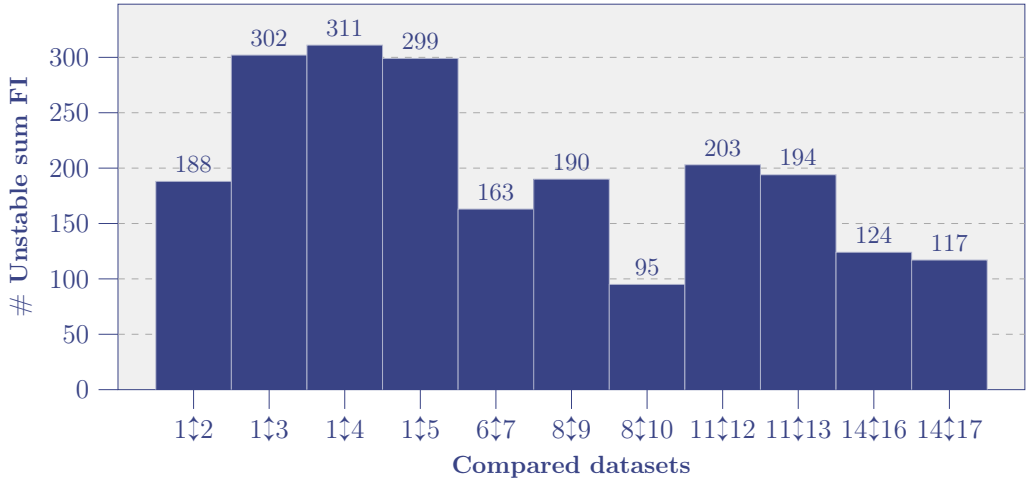

\paragraph*{Test 11}
\Cref{feature importance: fig: argmax datasets} shows how often each variable is within an $\epsilon$-bound of the largest FI in the dataset. Fully informative variables should attain the largest FI, according to \Cref{feature importance: prop: max Fi when fully determined}. In total, we observe that the fully informative variables are often the largest FI with respect to the other variables. However, there still remain many cases where they are not. \var{Test_11_failed_test} FI methods fail this test, thus definitively not having \Cref{feature importance: prop: max Fi when fully determined}. This makes interpretation difficult, when a variable can get more FI than a variable which fully determines the target variable. What does it mean, when a variable is more important than a variable that gives perfect information?

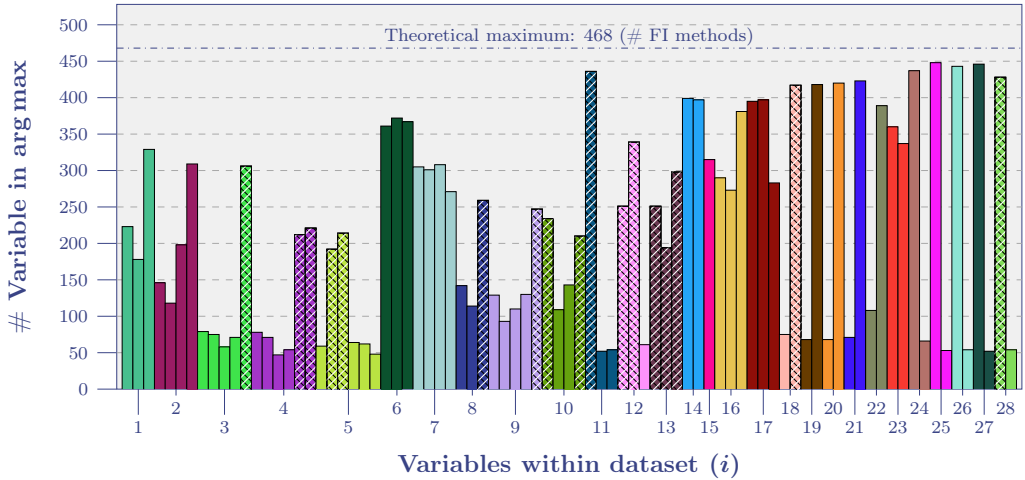
\begin{figure*}
    \resizebox*{\linewidth}{!}{
        \begin{tikzpicture}
            \def\maxx{14};
            \def\maxy{6};
            \def\ystep{88};
            \def\xstep{6};
            \def\barhelp{1/\xstep}
            \def\barwidth{\barhelp / 2};

            \draw[fill = backgroundgray, draw = boxcolor] (0,0) rectangle (\maxx, \maxy);

            \node[labelcolor] at (\maxx / 2, -1.2) {\bf \small Variables within dataset ($\bm{i}$)};
            \node[rotate = 90, labelcolor] at (-1.5, \maxy / 2) {\bf \small \# Variable in $\bm{\argmax}$};

            \foreach \y in {0,50,...,528} {
                    \node[anchor = east, labelcolor] at (-0.3,\y / \ystep){\scriptsize \y};
                    \draw[tickcolor] (-0.3, \y / \ystep) -- (0, \y / \ystep);

                }

            \foreach \y in {50, 100,...,528} {
                    \draw[dashed, backgrounddash] (0, \y / \ystep) -- (\maxx, \y / \ystep);
                }
            \draw[dashdotted, backgrounddash, scheme4] (0, 468 / \ystep) -- (\maxx, 468 / \ystep);
            \node[above = -0.1 cm, scheme4] at (\maxx / 2, 468 / \ystep) {\scriptsize Theoretical maximum: 468 (\# FI methods)};

            \foreach[count = \j, evaluate={
                        \NA=int(mod(\j,2));
                    }] \x/\i in {
                    2/1,
                    5.5/2,
                    10/3,
                    15.5/4,
                    21.5/5,
                    26/6,
                    29.5/7,
                    33/8,
                    37/9,
                    41.5/10,
                    45/11,
                    48/12,
                    51/13,
                    53.5/14,
                    55/15,
                    57/16,
                    60/17,
                    62.5/18,
                    64.5/19,
                    66.5/20,
                    68.5/21,
                    70.5/22,
                    72.5/23,
                    74.5/24,
                    76.5/25,
                    78.5/26,
                    80.5/27,
                    82.5/28
                }{

                    \ifthenelse{\NA = 0}{
                        \node[labelcolor] at (\x / \xstep, -0.3){\scriptsize \i};
                        \draw[tickcolor] (\x / \xstep, -0.1) -- (\x / \xstep, 0);
                    }
                    {
                        \node[labelcolor] at (\x / \xstep, -0.6){\scriptsize \i};
                        \draw[tickcolor] (\x / \xstep, -0.4) -- (\x / \xstep, 0);
                    }
                }

            \foreach[count = \x] \y / \colora / \colorb / \colorc in {
                    223 / 72 / 191 / 142,
                    178 / 72 / 191 / 142,
                    329 / 72 / 191 / 142,
                    146 / 153 / 28 / 100,
                    118 / 153 / 28 / 100,
                    198 / 153 / 28 / 100,
                    309 / 153 / 28 / 100,
                    79 / 63 / 227 / 75,
                    75 / 63 / 227 / 75,
                    58 / 63 / 227 / 75,
                    71 / 63 / 227 / 75,
                    306 / 63 / 227 / 75,
                    78 / 163 / 53 / 200,
                    71 / 163 / 53 / 200,
                    47 / 163 / 53 / 200,
                    54 / 163 / 53 / 200,
                    212 / 163 / 53 / 200,
                    221 / 163 / 53 / 200,
                    59 / 186 / 227 / 66,
                    192 / 186 / 227 / 66,
                    214 / 186 / 227 / 66,
                    64 / 186 / 227 / 66,
                    62 / 186 / 227 / 66,
                    48 / 186 / 227 / 66,
                    361 / 11 / 82 / 46,
                    372 / 11 / 82 / 46,
                    367 / 11 / 82 / 46,
                    305 / 161 / 207 / 207,
                    301 / 161 / 207 / 207,
                    308 / 161 / 207 / 207,
                    271 / 161 / 207 / 207,
                    142 / 50 / 61 / 150,
                    114 / 50 / 61 / 150,
                    259 / 50 / 61 / 150,
                    129 / 185 / 158 / 235,
                    93 / 185 / 158 / 235,
                    110 / 185 / 158 / 235,
                    130 / 185 / 158 / 235,
                    247 / 185 / 158 / 235,
                    234 / 101 / 161 / 14,
                    109 / 101 / 161 / 14,
                    143 / 101 / 161 / 14,
                    210 / 101 / 161 / 14,
                    436 / 8 / 87 / 130,
                    52 / 8 / 87 / 130,
                    54 / 8 / 87 / 130,
                    251 / 253 / 146 / 250,
                    339 / 253 / 146 / 250,
                    61 / 253 / 146 / 250,
                    251 / 104 / 55 / 79,
                    194 / 104 / 55 / 79,
                    298 / 104 / 55 / 79,
                    399 / 36 / 165 / 247,
                    397 / 36 / 165 / 247,
                    315 / 251 / 9 / 152,
                    290 / 230 / 195 / 82,
                    273 / 230 / 195 / 82,
                    381 / 230 / 195 / 82,
                    395 / 144 / 14 / 8,
                    397 / 144 / 14 / 8,
                    283 / 144 / 14 / 8,
                    75 / 253 / 181 / 172,
                    417 / 253 / 181 / 172,
                    68 / 104 / 60 / 0,
                    418 / 104 / 60 / 0,
                    68 / 246 / 147 / 46,
                    420 / 246 / 147 / 46,
                    71 / 63 / 22 / 249,
                    423 / 63 / 22 / 249,
                    108 / 127 / 136 / 97,
                    389 / 127 / 136 / 97,
                    360 / 247 / 57 / 49,
                    337 / 247 / 57 / 49,
                    437 / 180 / 114 / 107,
                    66 / 180 / 114 / 107,
                    448 / 250 / 27 / 252,
                    53 / 250 / 27 / 252,
                    443 / 141 / 228 / 211,
                    54 / 141 / 228 / 211,
                    446 / 25 / 79 / 70,
                    52 / 25 / 79 / 70,
                    428 / 130 / 220 / 89,
                    54 / 130 / 220 / 89
                }{
                    \definecolor{colorss}{RGB}{\colora, \colorb, \colorc}
                    \ifthenelse{\x = 12 \OR \x = 17 \OR \x = 18 \OR \x = 20 \OR \x = 21 \OR \x = 34 \OR \x = 39 \OR \x = 40 \OR \x = 43 \OR \x = 44 \OR \x = 47 \OR \x = 48 \OR \x = 50 \OR \x = 51 \OR \x = 52 \OR \x = 63 \OR \x = 82 }{
                        \draw[preaction = {fill = colorss, draw = black}, pattern = {north west lines[distance = 0.3pt]}, pattern color = black] (\x / \xstep - \barwidth, 0) -- (\x / \xstep - \barwidth, \y / \ystep) -- (\x / \xstep + \barwidth, \y / \ystep) -- (\x / \xstep + \barwidth, 0);
                        \draw[pattern = {north east lines[distance = 0.3pt, xshift = 0.15pt]}, pattern color = white] (\x / \xstep - \barwidth, 0) -- (\x / \xstep - \barwidth, \y / \ystep) -- (\x / \xstep + \barwidth, \y / \ystep) -- (\x / \xstep + \barwidth, 0);

                    }{
                        \draw[preaction = {fill = colorss}, draw = black] (\x / \xstep - \barwidth, 0) -- (\x / \xstep - \barwidth, \y / \ystep) -- (\x / \xstep + \barwidth, \y / \ystep) -- (\x / \xstep + \barwidth, 0);
                    }
                }

        \end{tikzpicture}
    }
    \caption{\textbf{Argmax FI:} For each variable in every dataset, we determine how often it receives the largest FI (within an $\epsilon$-bound for $\epsilon = 0.01$) with respect to the other variables in the dataset. Fully informative variables should attain the largest FI (see \Cref{feature importance: prop: max Fi when fully determined}). All fully informative variables are shaded in the figure.} \label{feature importance: fig: argmax datasets}

\end{figure*}

\paragraph*{Test 10, 17, 18}
These tests all evaluate if the FI method assigns a specific value to a feature. From \Cref{feature importance: tab: summary passed}, we observe that not many methods are able to pass these tests. This is not surprising, as they have not been thoroughly tested yet to give a specific value. This is one of the important contributions of this research, which is why we want to elaborate on the attempts that have been made in previous research. A lot of synthetic datasets for FI have been proposed \cite{Abe2005,Zien2009,Altmann2010,Stijven2011,Anh2012,Owen2016,Song2016,Shin2017,Carletti2017,Aas2019,Williamson2020,Lundberg2020,Giles2022,Onig2021,Hooker2018,Khodadadian2021,Lundberg2018,Lu2018,Frye2019,Merrick2019,Sundararajan2019,Dhamdhere2020,Tonekaboni2020,Casalicchio2018,Li2019,Molnar2021,Johnsen2021,Zhou2020,Fryer2021}, but no specific desirable FI values were given. Most commonly, synthetic datasets are generated to evaluate the ability of an FI method to find \emph{noisy} features \cite{Altmann2010,Stijven2011,Anh2012,Shin2017,Carletti2017,Williamson2020,Giles2022,Hooker2018,Sundararajan2019,Li2019,Johnsen2021,Zhou2020}. The common general concept of such a dataset is that the target variable is \emph{independent} of certain variables. The FI values are commonly evaluated by comparing the FI values of independent variables with dependent variables with the goal to establish if the FI method is able to find independent variables. If the FI method actually predicts the exact desirable FI is not considered. Next, we highlight the papers where some comment about the desired FI is made. Lundberg et al.~\cite{Lundberg2018} give two similar datasets, where one variable \emph{increases} in importance. They evaluate multiple FI methods to see if the same behavior is reflected in the outcome of these methods. This shows that some commonly used methods could assign lower importance to a variable, when it should actually be increasing. Giles et al.~\cite{Giles2022} also design multiple artificial datasets to represent different scenarios, where comments are made about which variables should obtain more FI. Sundararajan et al.~\cite{Sundararajan2019} remark that if every feature value is \emph{unique}, that all variables get \emph{equal} attributions for an FI method (CES) even if the function is not symmetric in the variables. If a tiny amount of noise is added to each feature, all features would get identical attributions. However, no assessment is done on the validity of this outcome. Owen et al.~\cite{Owen2016} give the following example. Let $f(x_1,x_2) = 10^6x_1 + x_2$ with $x_1 = 10^6x_2$, where they argue that, despite the larger variance of $x_1$, both variables are equally important, as the function can be written as a function of $x_1$ alone, but also only as a function of $x_2$. Although we have previously seen that `written as a function of' is not a good criterion (due to dependencies), we agree with the authors that the FI should be equal. Another example is given by Owen et al.~\cite{Owen2016}, where $\PP(x_1=0,x_2=0,y = y_0) = p_0$, $\PP(x_1=1,x_2=0,y = y_1) = p_1$, and $\PP(x_1=0,x_2=1,y = y_2) = p_2$ are the possible outcomes. If $p_0=0$, it is stated in \cite{Owen2016} that the Shapley relative importance of $x_1$ is $\frac{1}{2}$, which is \enquote{what it must be because there is then a bijection between $x_1$ and $x_2$}. This is an interesting observation, as most papers do not comment about the validity of an outcome. Additionally, when $y_1=y_2$ (and $y_0\neq y_1$), Owen et al.~\cite{Owen2016} argue that the most important variable, is the one with the largest variance. Fryer et al.~\cite{Fryer2021} also create a binary XOR dataset (see \Cref{feature importance: dataset: 14}). They evaluate seven FI methods for this specific dataset. The role of $X_1$ and $X_2$ is symmetric, thus the assigned FI should also be identical. It is shown that six out of seven methods do indeed give a symmetrical result. However, the exact FI value varies greatly. \emph{SHAP} gives FI of $3.19$, whereas \emph{Shapley DC} assigns $0.265$ as FI. Only symmetry is checked, not the accuracy of the FI method. In conclusion, existing research was not focussed on predicting the exact accurate FI values. It is therefore not surprising that FI methods fail these accuracy tests so often. \Cref{feature importance: tab: wrong specific value} outlines in more detail how often the variables are assigned an FI value outside an $\epsilon$-bound (with $\epsilon = 0.01$) of the desired outcome. With \Cref{feature importance: dataset: 11}, the FI methods mostly struggle with assigning 1 to the fully informative variable. In total, \var{Test_10_failed_test} methods failed \Cref{feature importance: test: 10}. For \Cref{feature importance: dataset: 14,feature importance: dataset: 17}, the two XOR variables fail about as often. Comparing these two datasets, it is interesting to note that the XOR variables fail more often, when a null-independent variable is added. In total, \var{Test_17_failed_test} methods failed \Cref{feature importance: test: 17}. \Cref{feature importance: test: 18} is hard, as the FI method should assign the correct values for all probability datasets (see \Cref{feature importance: appendix: datasets}). Only five methods are able to pass this test: {\slshape 152. mutualInformation}, {\slshape 153. roughsetConsistency}, {\slshape 162. RF}, {\slshape 175. fechner corr}, and {\slshape 202. \bpfi{}}. These five methods also pass \Cref{feature importance: test: 10}. However, besides \bpfi{}, there is only one method that also satisfies
\Cref{feature importance: test: 17}, which is {\slshape 162. RF}. The other three methods all assign only zeros for \Cref{feature importance: dataset: 14,feature importance: dataset: 17}, not identifying the value that the XOR variables hold, when their information is combined. In \Cref{feature importance: fig: breakdown probability datasets}, a breakdown is given for each probability dataset how often FI methods fail. An unexpected result, is that the dataset with probability $p<\frac{1}{2}$ and the dataset with probability $1-p$ do not fail as often. Consistently, $p<\frac{1}{2}$ fails less often than its counterpart $1-p$, although the datasets are the same up to a reordering of the features and the samples. This effect can also be seen in \Cref{feature importance: tab: wrong specific value}.

\begin{table}
    \begin{center}
        \caption{\textbf{Specific outcomes:} \Cref{feature importance: test: 10,feature importance: test: 17,feature importance: test: 18} all evaluate if an FI method gives a specific outcome for certain dataset. In this table, it is outlined how often each variable of these datasets is assigned a value outside an $\epsilon$-bound (with $\epsilon=0.01$) of the desired outcome.}
        \label{feature importance: tab: wrong specific value}
        \begin{adjustbox}{width=0.8\linewidth,keepaspectratio}
            \def\arraystretch{1.3}
            \begin{tabular}{LL*{6}{C}} \topprule
                \headcellcolor                                              & \headcellcolor                                                       & \multicolumn{6}{C}{\headcellcolor \headfont \textbf{\# Non desirable outcome}}                                                                           \\
                \headcellcolor                                              & \headcellcolor                                                       & \multicolumn{3}{C}{\subheadcellcolor \subheadfont not NaN}                     & \multicolumn{3}{C}{\subheadcellcolor \subheadfont NaN}                  \\
                \multirow{-3}{*}{\headcellcolor \headfont \textbf{Dataset}} & \multirow{-3}{*}{\headcellcolor \tchead{\headfont \textbf{Desirable}                                                                                                                                                            \\ {\headfont \textbf{outcome}}}} & \subheadcellcolor \subheadfont  $X_1$ & \subheadcellcolor \subheadfont $X_2$ & \subheadcellcolor \subheadfont $X_3$ & \subheadcellcolor \subheadfont $X_1$ & \subheadcellcolor \subheadfont $X_2$ & \subheadcellcolor \subheadfont $X_3$ \\ \middrule
                \trowcolor \trowfont \ref{feature importance: dataset: 11}  & $(1,0,0)$                                                            & 360                                                                            & 89                                                     & 88 & 4 & 4 & 4 \\
                \trowcolor \trowfont \ref{feature importance: dataset: 14}  & $(\frac{1}{2},\frac{1}{2})$                                          & 353                                                                            & 351                                                    & -  & 5 & 5 & - \\
                \trowcolor \trowfont \ref{feature importance: dataset: 17}  & $(\frac{1}{2},\frac{1}{2},0)$                                        & 369                                                                            & 364                                                    & 90 & 5 & 5 & 5 \\
                \trowcolor \trowfont \ref{feature importance: dataset: 18}  & $(0,1)$                                                              & 82                                                                             & 352                                                    & -  & 4 & 4 & - \\
                \trowcolor \trowfont \ref{feature importance: dataset: 19}  & $(\frac{1}{10},\frac{9}{10})$                                        & 412                                                                            & 434                                                    & -  & 3 & 3 & - \\
                \trowcolor \trowfont \ref{feature importance: dataset: 20}  & $(\frac{2}{10},\frac{8}{10})$                                        & 434                                                                            & 438                                                    & -  & 3 & 3 & - \\
                \trowcolor \trowfont \ref{feature importance: dataset: 21}  & $(\frac{3}{10},\frac{7}{10})$                                        & 435                                                                            & 441                                                    & -  & 3 & 3 & - \\
                \trowcolor \trowfont \ref{feature importance: dataset: 22}  & $(\frac{4}{10},\frac{6}{10})$                                        & 439                                                                            & 436                                                    & -  & 3 & 3 & - \\
                \trowcolor \trowfont \ref{feature importance: dataset: 23}  & $(\frac{5}{10},\frac{5}{10})$                                        & 423                                                                            & 422                                                    & -  & 3 & 3 & - \\
                \trowcolor \trowfont \ref{feature importance: dataset: 24}  & $(\frac{6}{10},\frac{4}{10})$                                        & 448                                                                            & 447                                                    & -  & 3 & 3 & - \\
                \trowcolor \trowfont \ref{feature importance: dataset: 25}  & $(\frac{7}{10},\frac{3}{10})$                                        & 449                                                                            & 446                                                    & -  & 3 & 3 & - \\
                \trowcolor \trowfont \ref{feature importance: dataset: 26}  & $(\frac{8}{10},\frac{2}{10})$                                        & 446                                                                            & 444                                                    & -  & 3 & 3 & - \\
                \trowcolor \trowfont \ref{feature importance: dataset: 27}  & $(\frac{9}{10},\frac{1}{10})$                                        & 444                                                                            & 435                                                    & -  & 3 & 3 & - \\
                \trowcolor \trowfont \ref{feature importance: dataset: 28}  & $(1,0)$                                                              & 352                                                                            & 86                                                     & -  & 5 & 5 & - \\
                \bottommrule
            \end{tabular}
        \end{adjustbox}
    \end{center}
\end{table}

\begin{figure*}
    \resizebox*{\linewidth}{!}{
        \begin{tikzpicture}
            \def\maxx{14};
            \def\maxy{6};
            \def\ystep{41};
            \def\xstep{0.86};
            \def\barwidth{0.581};

            \draw[decoration={discontinuity,amplitude=1ex},decorate, fill = backgroundgray, draw = boxcolor] (0,0) -- (0, 50 / \ystep)-- (0,\maxy)-- (\maxx, \maxy) -- (\maxx, 0) -- cycle;

            \node[anchor = east, labelcolor] at (-0.3,0 / \ystep){0};
            \draw[tickcolor] (-0.3, 0 / \ystep) -- (0, 0 / \ystep);
            \node[anchor = east, labelcolor] at (-0.3,50 / \ystep){300};
            \draw[tickcolor] (-0.3, 50 / \ystep) -- (0, 50 / \ystep);
            \node[anchor = east, labelcolor] at (-0.3,100 / \ystep){350};
            \draw[tickcolor] (-0.3, 100 / \ystep) -- (0, 100 / \ystep);
            \node[anchor = east, labelcolor] at (-0.3, 150 / \ystep){400};
            \draw[tickcolor] (-0.3, 150 / \ystep) -- (0, 150 / \ystep);
            \node[anchor = east, labelcolor] at (-0.3, 200 / \ystep){450};
            \draw[tickcolor] (-0.3, 200 / \ystep) -- (0, 200 / \ystep);

            \foreach \y in {50, 100, ..., 200}{
                    \draw[dashed, backgrounddash] (0, \y / \ystep) -- (\maxx, \y / \ystep);
                }

            \foreach[evaluate = {\x = int(10 * \j) + 1}] \j in {0.0, 0.1, ..., 1.1} {
                    \node[labelcolor] at (\x / \xstep , -0.6){\pgfmathprintnumber{\j}};
                    \draw[tickcolor] (\x / \xstep, -0.3) -- (\x / \xstep, 0);
                }

            \node[labelcolor] at (\maxx / 2, -1.2) {\bf \small Probability dataset $(p)$};
            \node[rotate = 90, labelcolor] at (-1.5, \maxy / 2) {\bf \small Frequency failed};

            \foreach[count = \x, evaluate = {\y = \j - 250}] \j in {399, 438, 441, 445, 442, 427, 451, 453, 450, 450, 405}{

                    \draw[fill = barcolor3, draw = outlinecolor2] (\x / \xstep - \barwidth, 0) -- (\x / \xstep - \barwidth, \y / \ystep) -- (\x / \xstep + \barwidth, \y / \ystep) -- (\x / \xstep + \barwidth, 0);
                    \node[above, labelcolor] at (\x / \xstep, \y / \ystep){\small \j};
                }

        \end{tikzpicture}
    }
    \caption{\textbf{Breakdown Test 18 per dataset:} In \Cref{feature importance: test: 18} an FI method needs to assign the correct FI values for every probability dataset (see \Cref{feature importance: appendix: datasets}). In this figure, we breakdown per dataset how often an FI method fails.} \label{feature importance: fig: breakdown probability datasets}

\end{figure*}
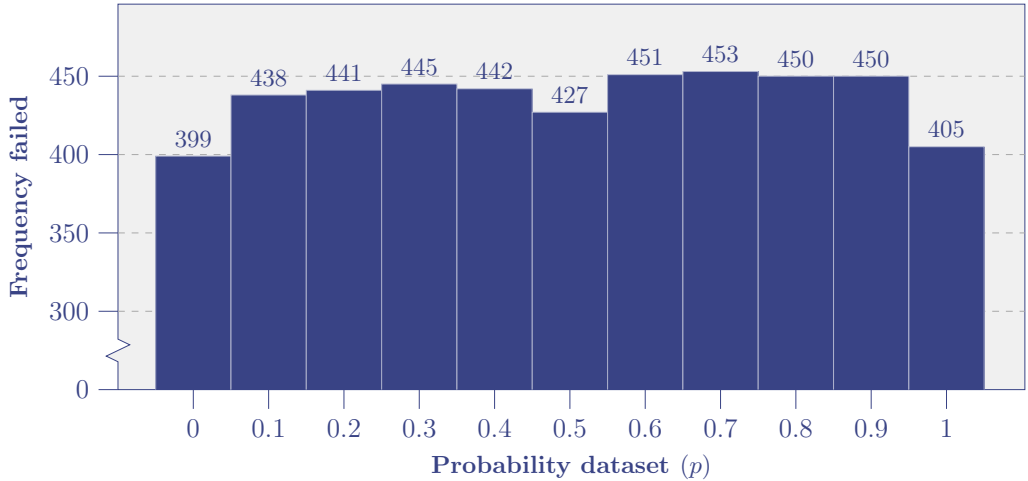

\paragraph*{No result}
Focussing on the \emph{no result} row of \Cref{feature importance: tab: summary passed}, there is one base method named \emph{158. KernelEstimator} in combination with \emph{Lasso} that in all cases did not work or exceeded running time. The large number of no results in \Cref{feature importance: test: 12} stem mostly from slow running times on the three datasets that are used in the test. At least 63 methods were too slow for each dataset, which automatically means that the test cannot be executed.

\section{Discussion and future research} \label{feature importance: sec: Discussion and future research}


Whilst it is recommended to use our new FI method, it is important to
understand the limitations and potential pitfalls. Below we elaborate on both the shortcomings of the approach proposed, and the related challenges for further research. We start by discussing by some matters that one needs to be aware of when applying the \bpfi{} (\Cref{feature importance: subsec: Creating awareness}). Next, we discuss some choices that were made for the experiments in \Cref{feature importance: subsec: Experimental design choices}. Finally, we elaborate on other possible research avenues in \Cref{feature importance: subsec: Additional matters}.

\subsection{Creating awareness} \label{feature importance: subsec: Creating awareness}

\paragraph*{Binning}
Berkelmans et al.~\cite{BP} explained that the way in which continuous data is discretized can have a considerable effect on the \bpdep, which is why all datasets that were used in our research are \emph{discrete}. If a feature has too many unique values (due to poor binning), it will receive a higher FI from \bpfi{}, as more information can be stored in the unique values (see \Cref{feature importance: prop: limiting outcome space}). On the other hand, when too few bins are chosen, an important feature can receive low FI, as the information is lost due to the binning. Future research should investigate and test which binning algorithms give the closest results to the underlying FI.

\paragraph*{Too few samples}
Consider the following dataset: $X_i, Y\sim\UU{\{0,1,\dots, 9\}}$ i.i.d. for $i\in \{1,\dots, 5\}$. Note that all features are null-independent, as $Y$ is just uniformly drawn without considering the features in any way. If $\nsamples = \infty$, the desired outcome would therefore be $\left (0,0,0,0,0 \right)$. However, when \emph{not enough} samples are given in the dataset, the features will get nonzero FI. Considering that the total number of different feature values is  $10^5$, combining all features \emph{does} actually give information about $Y$, when $\nsamples \ll 10^5$. For any possible combination of features, it is unlikely that it occurs more than once in the dataset. Therefore, knowing all feature values would (almost surely) determine the value of $Y$. \Cref{feature importance: prop: efficiency} gives that the sum of all FI should therefore be one. All feature variables are also \emph{symmetric} (\Cref{feature importance: prop: symmetry}), which is why the desired outcome is $(\frac{1}{5}, \frac{1}{5}, \frac{1}{5}, \frac{1}{5}, \frac{1}{5})$ instead. This example shows that one should be aware of the influence of the number of samples on the resulting FI. Variables that do \emph{not} influence $Y$ can still contain information, when not enough samples are provided. In this way, insufficient samples could lead to wrong conclusions, if one is not wary of this phenomenon.

\paragraph*{Counterintuitive dependency case}
The \emph{Berkelmans-Pries} dependency of $Y$ on $X$ measures how much probability mass of $Y$ is shifted by knowing $X$. However, two similar shifts in probability mass could lead to different predictive power. To explain this, we examine the following dataset. $X_1,X_2\sim\UU{\{0,1\}}$ with \begin{align*}
    \PP(Y=y \vert X_1 = x_1, X_2 = x_2) & = \left \{
    \begin{array}{ll}
        1/4 & \text{if } (x_2,y) = (0,0),        \\[1ex]
        3/4 & \text{if } (x_2,y) = (0,1),        \\[1ex]
        5/8 & \text{if } (x_1, x_2,y) = (0,1,0), \\[1ex]
        3/8 & \text{if } (x_1, x_2,y) = (0,1,1), \\[1ex]
        7/8 & \text{if } (x_1, x_2,y) = (1,1,0), \\[1ex]
        1/8 & \text{if } (x_1, x_2,y) = (1,1,1).
    \end{array}
    \right .
\end{align*}
Knowing the value of $X_2$ shifts the distribution of $Y$. Before, $Y$ was split 50/50, but when the value of $X_2$ is known, the labels are either split 25/75 or 75/25, depending on the value of $X_2$. Knowing $X_1$ gives even more information, as e.g., knowing $X_1 = X_2=1$ makes it more likely that $Y=0$. However, the shift in distribution of $Y$ is the same for knowing only $X_2$ and $X_1$ combined with $X_2$, which results in $\Dep{X_2}{Y} = \Dep{X_1\cup X_2}{Y}$. This is a counterintuitive result. Globally, knowing $X_2$ or $X_1\cup X_2$ gives the same shift in distribution, but locally we can predict $Y$ much better if we know $X_1$ as well. We are unsure how this effects the \bpfi{}. In this case, it follows that $\FI{X_1\cup X_2} > \FI{X_2}$, which is desirable. It is not unthinkable that a solution can be found to modify the dependency function in order to get a more intuitive result for such a case. Think e.g., of a different distance metric, that incorporates the local accuracy given the feature values or a conditional variant, which not only tests for independence, but also for conditional independence. These are all critical research paths that should be investigated.

\paragraph*{Using FI for feature selection}
\emph{Feature selection} (FS) is \enquote{the problem of choosing a small subset of features that ideally is necessary and sufficient to describe the target concept} \cite{Kira1992}. Basically, the objective is to find a subset of all features that gives the best performance for a given model, as larger feature sets could decrease the accuracy of a model \cite{Kursa2010}. Many FI methods actually stem from a FS procedure. However, it is important to stress that \emph{high} FI means that it should automatically be selected as feature. Shared knowledge with other features could render the feature less useful than expected. The other way around, \emph{low} FI features should not automatically be discarded. In combination with other features, it could still give some additional insights that other features are not able to provide.
Calculation of \bpfi{} values could also provide insight into which group of $K$ features $Y$ is most dependent on. To derive the result of \bpfi{}, all dependencies of $Y$ on a subset $S\subseteq \FISET$ are determined. If only $K$ variables are selected, it is natural to choose \begin{align*}
    S^*_K \in \argmax_{S \subseteq \FISET: |S| = K}\{\Dep{S}{Y}\}.
\end{align*}
These values are stored as an intermediate step in \bpfi{}, thus $S_K^*$ can be derived quickly thereafter.

\paragraph*{Larger outcome space leads to higher FI}
We have proven that a larger outcome space can never lead to a decrease in FI for \bpfi{}. This means, that features with more possible outcomes are more likely to attain a higher FI, depending on the distribution. There is a difference between a feature that has many possible outcomes that are \emph{almost never} attained, and a feature where many possible outcomes are \emph{regularly} observed. We do not find this property undesirable, as some articles suggest \cite{Zhou2020, Strobl2007}, as we would argue that a feature \emph{can} contain more information by storing the information in additional outcomes, which would lead to an non-decreasing FI.

\subsection{Experimental design choices} \label{feature importance: subsec: Experimental design choices}

\paragraph*{Regression}
To avoid binning issues, we only considered classification models and datasets. There are many more regression FI methods, that should be considered in a similar fashion. However, to draw clear and accurate conclusions, it is first necessary to understand how binning affects the results. Sometimes counterintuitive results can occur due to binning, that are not necessarily wrong. In such a case, it is crucial that the FI method is not depreciated.

\paragraph*{Runtime}
In the experiments, it could happen that an FI method had \emph{no result}, due to an excessive runtime or incompatible FI scores. The maximum runtime for each algorithm was set to one hour per dataset on an i7-12700K processor with 4 algorithms running simultaneously. The maximum runtime was necessary due to the sheer number of FI methods and datasets. Running four algorithms in parallel could unfairly penalize the runtime, as the processor is sometimes limited by other algorithms. In some occurrences, other parallel processes were already finished, which could potentially lower the runtime of an algorithm. There is a potential risk here, that accurate (but slow) FI methods are not showing up in the results. However, our synthetic datasets are relatively small with respect to the number of samples and the number of features, and we argue that one hour should be reasonable. Depending on the use case, sometimes a long time can be used to determine an FI value, whereas in other cases it could be essential to determine it rather quickly. Especially for larger datasets, it could even be unfeasible to run some FI methods. \bpfi{} uses Shapley values, which are exponentially harder to compute when the number of features grow. Approximation algorithms should be developed to faster estimate the true \bpfi{} outcome. Quick approximations could be useful if the runtime is much faster and the approximation is decent enough. Already, multiple papers have suggested approaches to approximate Shapley values faster \cite{Aas2019,Johnsen2021,Lipovetsky2001,Kononenko2014,Castro2009}. These approaches save time, but at what cost? A study could be done to find the best FI method given a dataset and an allowed running time.

\paragraph*{Stochasticity methods}
One factor we did not incorporate, is the \emph{stochasticity} of some FI methods. Some methods do not predict the same FI values, when it is repeatedly used. As example, {\slshape 79. rf} predicted for \Cref{feature importance: dataset: 3} (12.1, 11.7, 17.9, 15.2, 37.7) rounded to the first decimal. Running the method again gives a different result: (11.4, 12.0, 17.4, 15.6, 37.1), as this method uses a stochastic \emph{random forest}. In principle, it is \emph{undesirable} that an FI method is stochastic, as we believe that there should be a unique assignment of FI given a dataset. Due to the number of FI methods and datasets, we did not repeat and averaged each FI method. This would however give a better view on the performance of stochastic FI methods.

\paragraph*{Parameter tuning}
All FI methods were used with \emph{default} parameter values. Different parameter values could lead to more or less failed tests. However, the \emph{ideal} parameter setting is not known beforehand, making it necessary to search a wide range of parameters. This was not the focus of our research, but future research could try to understand and learn which parameter values should be chosen for a given dataset.

\paragraph*{Ranking FI methods}
In \Cref{feature importance: tab: best 20}, the 20 FI methods that passed the most tests were highlighted. However, it is important to stress that not every test is equally difficult. Depending on the user, some properties could be more or less relevant. It is e.g., much harder to accurately predict the specific values for 11 datasets (\Cref{feature importance: test: 18}), than to always predict non-negatively (\Cref{feature importance: test: 4}). Every test is weighed equally, but this does not necessarily represent the difficulty of passing each test accurately. However, we note that {\slshape 175. fechner corr} is the only FI method that passed \Cref{feature importance: test: 18}, that ended up outside the top 20. We stress that we focussed on finding out if FI methods adhere to the properties, not necessarily finding the best and most fair ranking.

\subsection{Additional matters} \label{feature importance: subsec: Additional matters}

\paragraph*{Global vs. local}
\bpfi{} is designed to determine the FI \emph{globally}. However, another important research area focusses on \emph{local} explanations. These explanations should provide information about why a specific sample has a certain target value instead of a different value. They provide the necessary interpretability that is increasingly demanded for practical applications. This could give insights for questions like: `If my income would be higher, could I get a bigger loan?', `Does race play a role in this prediction?', and `For this automated machine learning decision, what were the critical factors?'. Many local FI methods have been proposed, and some even use Shapley values. A structured review should be made about all proposed local methods, similar to our approach for global FI methods to find which local FI methods actually produce accurate explanations.

\bpfi{} can be modified to provide local explanations. For example, we can make the characteristic function localized in the following way.
Let $Y_{S,z}$ be $Y$ restricted to the event that $X_i=z_i$ for $i\notin S$, let us similarly define $X_{S,z}$. Then, we can define a localized characteristic function by:
\begin{equation}
    v_z(S):=\Dep{X_{S,z}}{Y_{S,z}}.
\end{equation}
When dealing with continuous data, assuming equality could be too strict. In this case, a precision vector parameter $\epsilon$ can be used, where we define $Y_{S,z,\epsilon}$ to be $Y$ restricted to the event that $|X_i-z_i|\leq \epsilon_i$ for $i\notin S$, and in the same way we define $X_{S,z,\epsilon}$. We then get the following localized characteristic function:
\begin{align*}
    v_{z,\epsilon}(S):=\Dep{X_{S,z,\epsilon}}{Y_{S,z,\epsilon}}.
\end{align*}
Additionally, there are at least two possible ways how \bpfi{} can be adapted to be used for local explanations if some distance function $d(i,j)$ and parameter $\delta$ are available to determine if sample $j$ is close enough to $i$ to be considered `local'. We can \begin{enumerate*}[label =(\Roman*)]
    \item discard all samples where $d(i,j) > \delta$ and/or
    \item generate samples, such that $d(i,j) \leq \delta$ for all generated samples.
\end{enumerate*}
Then, we can use \bpfi{} on the remaining samples and/or the generated samples, which would give local FI. Note that there should still be enough samples, as we have previously discussed that too few samples could lead to different FI outcomes. However, there are many more ways how \bpfi{} can be modified to be used for local explanations.

\paragraph*{Model-specific FI}
\bpfi{} is in principle model-agnostic, as the FI is determined of the dataset, not the FI for a prediction model. However, \bpfi{} can still provide insights for any specific model. By replacing the target variable with the predicted outcomes of the model, we can apply \bpfi{} to this new dataset, which gives insight into which features are useful in the prediction model. Additionally, one can compare these FI results with the original FI (before replacing the target variable with the predicted outcomes) to see in what way the model changed the FI.

\paragraph*{Additional properties}
In this research, we have proven properties of \bpfi{}. However, an in-depth study could lead to finding more useful properties. This holds both for \bpfi{} as well as the dependency function it is based on. Applying isomorphisms e.g., does not change the dependency function. Therefore, the \bpfi{} is also stable under isomorphisms. Understanding what properties \bpfi{} has is a double-edged sword. Finding useful properties shows the power of \bpfi{} and finding undesirable behavior could lead to a future improvement.

\paragraph*{Additional datasets}
Ground truths are often unknown for FI. In this research, we have given two kinds of datasets where the desirable outcomes are natural. It would however, be useful to create a \emph{larger} collection of datasets both for global and local FI with an exact ground truth. We recognize that this could be a tall order, but we believe that it is essential to further improve FI methods.

\paragraph*{Human labeling}
In some articles \cite{Shap,Ribeiro2016}, humans are used to evaluate explanations. An intriguing question to investigate is if humans are good at predicting FI. The \bpfi{} can be used as baseline to validate the values that are given by the participants. Are humans able to identify the correct order of FI? Even more difficult, can they predict close to the actual FI values?

\section{Summary} \label{feature importance: sec: Summary}


We started by introducing a novel FI method named \emph{Berkelmans-Pries} FI (\bpfi{}), which combines \emph{Shapley} values and the \emph{Berkelmans-Pries} dependency function \cite{BP}. In \Cref{feature importance: sec: Properties of BP-FI}, we proved many useful properties of \bpfi{}. We discussed which FI methods already exist and introduced datasets to evaluate if these methods adhere to the same properties. In \Cref{feature importance: subsec: Property evaluation}, we explain how the properties are tested. The results show that \bpfi{} is able to pass \emph{many more} tests than any other FI method from a large collection of FI methods (468), which is a significant step forwards. Most methods have not previously been tested to give exact results due to missing ground truths. In this research, we provide several specific datasets, where the desired FI can be derived. From the tests, it follows that previous methods are not able to accurately predict the desired FI values. In \Cref{feature importance: sec: Discussion and future research}, we extensively discussed the shortcomings of this paper, and the challenges for further research. There are many challenging research opportunities that should be explored to further improve interpretability and explainability of datasets and machine learning models.

\appendix
    \section{Datasets}  \label{feature importance: appendix: datasets}

    In this appendix, we discuss how the datasets are generated that are used in the experiments. We use \emph{fixed draw} instead of \emph{uniformly random} to draw each dataset \emph{exactly} according to its distribution. This is done to remove stochasticity from the dataset in order to get precise and interpretable results. An example of the difference between fixed draw and uniformly random can be seen in \Cref{feature importance: tab: difference drawn uniform}. The datasets consist of 1{,}000 samples, except for \Cref{feature importance: dataset: 6,feature importance: dataset: 7} which contains 2{,}000 samples to ensure null-independence. The datasets are designed to be computationally inexpensive, whilst still being able to test many properties (see \Cref{feature importance: subsec: Property evaluation}). Below, we outline the formulas that are used to generate the datasets and give the corresponding FI values of our novel method \bpfi{}.



    \datasetappendix{Binary system}{$X_i\sim\UU{\{0,1\}}$ i.i.d. for $i \in \{1,2,3\}$}{$Y := \sum_{i=1}^{3} 2^{i-1} \cdot X_i$}{$(X_1, X_2,X_3)$}{$\left (0.333, 0.333, 0.333 \right )$}


    \datasetappendix{Binary system with clone}{$X_i\sim\UU{\{0,1\}}$ i.i.d. for $i \in \{1,2,3\}$ and $X_1^{\text{clone}}:= X_1$.}{$Y := \sum_{i=1}^{3} 2^{i-1} \cdot X_i$}{$(X_1^{\text{clone}}, X_1, X_2,X_3)$}{$\left (0.202, 0.202, 0.298, 0.298 \right )$}


    \datasetappendix{Binary system with clone and one fully informative variable}{$X_i\sim\UU{\{0,1\}}$ i.i.d. for $i \in \{1,2,3\}$ and $X_1^{\text{clone}}:= X_1$ and $X_4^{\text{full}}:= Y^2$.}{$Y := \sum_{i=1}^{3} 2^{i-1} \cdot X_i$}{$(X_1^{\text{clone}}, X_1, X_2,X_3, X_4^{\text{full}})$}{$\left ( 0.148, 0.148, 0.183, 0.183, 0.338 \right )$}


    \datasetappendix{Binary system with clone and two fully informative variables}{$X_i\sim\UU{\{0,1\}}$ i.i.d. for $i \in \{1,2,3\}$ and $X_1^{\text{clone}}:= X_1$ and $X_4^{\text{full}}:= Y^2$, $X_5^{\text{full}}:= Y^3$.}{$Y := \sum_{i=1}^{3} 2^{i-1} \cdot X_i$}{$(X_1^{\text{clone}}, X_1, X_2,X_3, X_4^{\text{full}}, X_5^{\text{full}})$}{$\left ( 0.117, 0.117, 0.136, 0.136, 0.248, 0.248 \right )$}


    \datasetappendix{Binary system with clone and two fully informative variables different order}{$X_i\sim\UU{\{0,1\}}$ i.i.d. for $i \in \{1,2,3\}$ and $X_1^{\text{clone}}:= X_1$ and $X_4^{\text{full}}:= Y^2$, $X_5^{\text{full}}:= Y^3$.}{$Y := \sum_{i=1}^{3} 2^{i-1} \cdot X_i$}{$(X_3, X_4^{\text{full}}, X_5^{\text{full}}, X_1^{\text{clone}}, X_1, X_2)$}{$\left ( 0.136, 0.248, 0.248, 0.117, 0.117, 0.136 \right )$}


    \datasetappendix{Null-independent system}{$X_i^{\text{null-indep.}} \sim\UU{\{0,1\}}$ i.i.d. for $i\in \{1,2, 3\}$.}{$Y\sim\UU{\{0,1\}}$}{$(X_1^{\text{null-indep.}},X_2^{\text{null-indep.}},X_3^{\text{null-indep.}})$}{$\left ( 0.000, 0.000, 0.000 \right )$}


    \datasetappendix{Null-independent system with constant variable}{$X_i^{\text{null-indep.}} \sim\UU{\{0,1\}}$ i.i.d. for $i\in \{1,2, 3\}$ and $X_4^{\text{const, null-indep.}}:= 1$.}{$Y\sim\UU{\{0,1\}}$}{$(X_1^{\text{null-indep.}},X_2^{\text{null-indep.}},X_3^{\text{null-indep.}}, X_4^{\text{const, null-indep.}})$}{$\left ( 0.000, 0.000, 0.000, 0.000 \right )$}


    \datasetappendix{Uniform system increasing bins}{Let $\mathcal{L}_i:= \{0, 1 / (i-1), \dots, 1 \}$ be an equally spaced set. Define:
    \begin{align*}
        X_1^{\text{bins}=10} & := \argmax_{x_1 \in \mathcal{L}_{10}} \{Y \geq x_1\}, \\[2ex]  X_2^{\text{bins}=50} &:= \argmax_{x_2 \in \mathcal{L}_{50}} \{Y \geq x_2\}, \\[2ex] X_3^{\text{bins}=1{,}000,~\text{full}} &:= \argmax_{x_3 \in \mathcal{L}_{1{,}000}} \{Y \geq x_3\}.
    \end{align*}
    }{$Y \sim \UU{\mathcal{L}_{1{,}000}}$}{$(X_1^{\text{bins}=10}, X_2^{\text{bins}=50}, X_3^{\text{bins}=1{,}000,~\text{full}})$}{$\left ( 0.297, 0.342, 0.361 \right )$}


    \vspace*{0.5 cm}
    \datasetappendix{Uniform system increasing bins more variables}{Let $\mathcal{L}_i:= \{0, 1 / (i-1), \dots, 1 \}$ be an equally spaced set. Define:
    \begin{align*}
        X_1^{\text{bins}=10}                   & := \argmax_{x_1 \in \mathcal{L}_{10}} \{Y \geq x_1\},      \\[2ex]
        X_2^{\text{bins}=20}                   & := \argmax_{x_2 \in \mathcal{L}_{20}} \{Y \geq x_2\},      \\[2ex]
        X_3^{\text{bins}=50}                   & := \argmax_{x_3 \in \mathcal{L}_{50}} \{Y \geq x_3\},      \\[2ex]
        X_4^{\text{bins}=100}                  & := \argmax_{x_4 \in \mathcal{L}_{100}} \{Y \geq x_4\},     \\[2ex]
        X_5^{\text{bins}=1{,}000,~\text{full}} & := \argmax_{x_5 \in \mathcal{L}_{1{,}000}} \{Y \geq x_5\}.
    \end{align*}
    }{$Y \sim \UU{\mathcal{L}_{1{,}000}}$}{$(X_1^{\text{bins}=10}, X_2^{\text{bins}=20}, X_3^{\text{bins}=50}, X_4^{\text{bins}=100}, X_5^{\text{bins}=1{,}000,~\text{full}})$}{$\left ( 0.179, 0.193, 0.204, 0.208, 0.216 \right )$}

    \vfill
    \clearpage


    \datasetappendix{Uniform system increasing bins with clone different order}{Let $\mathcal{L}_i:= \{0, 1 / (i-1), \dots, 1 \}$ be an equally spaced set. Define:
    \begin{align*}
        X_1^{\text{bins}=10}                   & := \argmax_{x_1 \in \mathcal{L}_{10}} \{Y \geq x_1\},      \\[2ex]
        X_2^{\text{bins}=50}                   & := \argmax_{x_2 \in \mathcal{L}_{50}} \{Y \geq x_2\},      \\[2ex]
        X_3^{\text{bins}=1{,}000,~\text{full}} & := \argmax_{x_3 \in \mathcal{L}_{1{,}000}} \{Y \geq x_3\}, \\[2ex]
        X_3^{\text{clone},~\text{full}}        & := X_3^{\text{bins}=1{,}000,~\text{full}}.
    \end{align*}
    }{$Y \sim \UU{\mathcal{L}_{1{,}000}}$}{$(X_3^{\text{bins}=1{,}000,~\text{full}}, X_2^{\text{bins}=50}, X_1^{\text{bins}=10}, X_3^{\text{clone},~\text{full}})$}{$\left ( 0.262, 0.253, 0.223, 0.262 \right )$}


    \datasetappendix{Dependent system: 1x fully informative variable}{$X_1^{\text{full}}, X_2^{\text{null-indep.}}, X_3^{\text{null-indep.}} \sim \UU{\{1,2\}}$.}{$Y:= X_1^{\text{full}}$}{$(X_1^{\text{full}}, X_2^{\text{null-indep.}}, X_3^{\text{null-indep.}})$}{$\left ( 1.000, 0.000, 0.000 \right )$}


    \datasetappendix{Dependent system: 2x fully informative variable}{$X_1^{\text{full}}, X_3^{\text{null-indep.}} \sim \UU{\{1,2\}}$ and $X_2^{\text{full}}:= Y^2$.}{$Y:= X_1^{\text{full}}$}{$(X_1^{\text{full}}, X_2^{\text{full}}, X_3^{\text{null-indep.}})$}{$\left ( 0.500, 0.500, 0.000 \right )$}


    \datasetappendix{Dependent system: 3x fully informative variable}{$X_1^{\text{full}} \sim \UU{\{1,2\}}$ and $X_2^{\text{full}}:= Y^2$, $X_3^{\text{full}}:= Y^3$.}{$Y:= X_1^{\text{full}}$}{$(X_1^{\text{full}}, X_2^{\text{full}}, X_3^{\text{full}})$}{$\left ( 0.333, 0.333, 0.333 \right )$}


    \datasetappendix{XOR dataset}{$X_1, X_2 \sim \UU{\{1,2\}}$.}{$Y:= X_1 \cdot (1-X_2) + X_2 \cdot (1-X_1)$}{$(X_1, X_2)$}{$\left ( 0.500, 0.500 \right )$}


    \datasetappendix{XOR dataset one variable}{$X_1^{\text{null-indep.}}\sim \UU{\{1,2\}}$.}{$Y:= X_1^{\text{null-indep.}} \cdot (1-X_2) + X_2 \cdot (1-X_1^{\text{null-indep.}})$ with $ X_2 \sim \UU{\{1,2\}}$}{$(X_1^{\text{null-indep.}})$}{$\left ( 0.000 \right )$}


    \datasetappendix{XOR dataset with clone}{$X_1, X_2 \sim \UU{\{1,2\}}$ and $X_1^{\text{clone}}:= X_1$.}{$Y:= X_1 \cdot (1-X_2) + X_2 \cdot (1-X_1)$}{$(X_1^{\text{clone}}, X_1, X_2)$}{$\left (0.167, 0.167, 0.667 \right )$}


    \datasetappendix{XOR dataset with null independent}{$X_1, X_2 \sim \UU{\{1,2\}}$ and $X_3^{\text{null-indep.}} \sim \UU{\{0,3\}}$.}{$Y:= X_1 \cdot (1-X_2) + X_2 \cdot (1-X_1)$}{$(X_1, X_2, X_3^{\text{null-indep.}})$}{$\left ( 0.500, 0.500, 0.000 \right )$}


    \datasettappendix{Probability datasets}{$X_i = Z_i + S$ with $Z_i \sim\UU{\{0,2\}}$ i.i.d. for $i = 1,2$ and $\PP(S=1) = p$, $\PP(S=2) = 1- p$.}{$Y = \lfloor X_S / 2 \rfloor$}{$(X_1, X_2)$}{$\left ( p, 1-p \right )$}

    \section{Tests} \label{feature importance: appendix: tests}

    This appendix gives an overview of the tests that are used for each FI method to evaluate if they adhere to the properties given in \Cref{feature importance: sec: Properties of BP-FI}. Most tests are straightforward, but additional explanations are given in \Cref{feature importance: subsec: Property evaluation}.

    \testappendix
    {Efficiency sum \bpfi{}}
    {\Cref{feature importance: prop: efficiency}}
    {We evaluate if the sum of all FI is equal to the sum of the \emph{Berkelmans-Pries} dependency function of $Y$ on all features. When an FI value of NaN or infinite is assigned, the sum is automatically not equal to the sum for \bpfi{}.}

    \testappendix
    {Efficiency stable}
    {\Cref{feature importance: corollary: sum stable}}
    {Whenever a variable is added to a dataset, we examine if the sum of all FI changes. If a variable does not give any additional information compared to the other variables, the sum of all FI should stay the same.}

    \testappendix
    {Symmetry}
    {\Cref{feature importance: prop: symmetry}}
    {In some datasets, there are \emph{symmetrical} variables (see \Cref{feature importance: prop: symmetry}). We determine for all symmetrical variables if they receive identical FI.}

    \testappendix
    {Range (lower)}
    {\Cref{feature importance: prop: range}}
    {We examine for all FI outcomes if they are greater or equal to zero.}

    \testappendix
    {Range (upper)}
    {\Cref{feature importance: prop: range}}
    {We examine for all FI outcomes if they are smaller or equal to one.}

    \testappendix
    {Bounds \bpfi{} (lower)}
    {\Cref{feature importance: prop: bounds}}
    {We evaluate if the bounds given in \Cref{feature importance: prop: bounds} also hold for other FI methods. Every $\FI{X}$ with $X\in \FISET$ can be lower bounded for \bpfi{} by $\frac{\Dep{X}{Y}}{\nvariables} \leq \FI{X}.$}

    \testappendix
    {Bounds \bpfi{} (upper)}
    {\Cref{feature importance: prop: bounds}}
    {We evaluate if the bounds given in \Cref{feature importance: prop: bounds} also hold for other FI methods. Every $\FI{X}$ with $X\in \FISET$ can be upper bounded for \bpfi{} by $
            {X} \leq \Dep{\FISET}{Y}.
        $}

    \testappendix
    {Null-independent implies zero FI}
    {\Cref{feature importance: prop: zero FI}}
    {In some datasets, there are \emph{null-independent} variables. In these cases, we investigate if they also receive zero FI.}

    \testappendix
    {Zero FI implies null-independent}
    {\Cref{feature importance: prop: zero FI}}
    {When a variable gets zero FI, it should hold that such a feature is null-independent.}

    \testappendix
    {One fully informative, two null-independent}
    {\Cref{feature importance: prop: FI equal to one}}
    {
    {feature importance: appendix: datasets}) consists of a fully dependent target variable $Y:= X_1^{\text{full}}$ and two null-independent variables $X_2^{\text{null-indep.}}, X_3^{\text{null-indep.}}$. We test if $\FI{ X_1^{\text{full}}} =1 $ and $\FI{X_2^{\text{null-indep.}}} = \FI{X_3^{\text{null-indep.}}} = 0.$
    }

    \testappendix
    {Fully informative variable in argmax FI}
    {\Cref{feature importance: prop: max Fi when fully determined}}
    {Whenever a fully informative feature exists in a dataset, there should not be a feature that attains a higher FI.}

    \testappendix
    {Limiting the outcome space}
    {\Cref{feature importance: prop: limiting outcome space}}
    {
        To evaluate if applying a measurable function $f$ to a RV $X$ could increase the FI, we examine the datasets where the same RV is binned using different bins. The binning can be viewed as applying a function $f$. Whenever less bins are used, the FI should not increase.}

    \testappendix
    {Adding features can increase FI}
    {\Cref{feature importance: prop: adding features can increase FI}}
    {Whenever a feature is added to a dataset, we examine if this ever increases the FI of an original variable. If the FI never increases, we consider this a fail.}

    \testappendix
    {Adding features can decrease FI}
    {\Cref{feature importance: prop: adding features can decrease FI}}
    {Whenever a feature is added to a dataset, we examine if this ever decreases the FI of an original variable. If the FI never decreases, we consider this a fail.}

    \testappendix
    {Cloning does not increase FI}
    {\Cref{feature importance: prop: cloning does not increase FI}}
    {We evaluate if adding a clone to a dataset increase the FI of the original variable.}

    \testappendix
    {Order does not change FI}
    {\Cref{feature importance: prop: order does not change FI}}
    {We check if the order of the variables changes the assigned FI.}

    \testappendix{Outcome XOR}{\Cref{feature importance: prop: xor dataset}}{This test evaluates the specific outcome of two datasets. For \Cref{feature importance: dataset: 14} the desired outcome is $(1/2, 1/2)$ and $(1/2, 1/2, 0)$ for \Cref{feature importance: dataset: 17}. An FI method fails this test when one of the FI values falls outside the $\epsilon$-bound of the desired outcome.}

    \testappendix
    {Outcome probability datasets}
    {\Cref{feature importance: prop: probability dataset}}
    {
        This test evaluates the specific outcomes of all probability datasets (\Cref{feature importance: dataset: 18,feature importance: dataset: 19,feature importance: dataset: 20,feature importance: dataset: 21,feature importance: dataset: 22,feature importance: dataset: 23,feature importance: dataset: 24,feature importance: dataset: 25,feature importance: dataset: 26,feature importance: dataset: 27,feature importance: dataset: 28}). The desired outcome for probability $p$ is $(p, 1-p)$. An FI method fails this test when one of the FI values falls outside the $\epsilon$-bound of the desired outcome.
    }

    \begingroup
    \setlength{\emergencystretch}{1.67em} 
    \printbibliography
    
    \endgroup

\end{document}